\documentclass[twoside,11pt]{article}

\usepackage{jmlr2e}
\usepackage{bm}

\usepackage{thmtools}
\usepackage{thm-restate}
\declaretheorem[name=Assumption]{ass}
\usepackage{amsmath}
\usepackage{amssymb}
\usepackage{mathtools}
\newcommand{\argmin}{\operatornamewithlimits{argmin}}
\newcommand{\argmax}{\operatornamewithlimits{argmax}}
\newcommand{\arginf}{\operatornamewithlimits{arginf}}

\makeatletter
\newcommand\blfootnote[1]{%
  \begingroup
  \renewcommand{\@makefntext}[1]{\noindent\makebox[1.8em][r]#1}
  \renewcommand\thefootnote{}\footnote{#1}%
  \addtocounter{footnote}{-1}%
  \endgroup
}
\makeatother

\jmlrheading{1}{}{}{}{}{}

\firstpageno{1}

\begin{document}

\title{Pure Exploration under Mediators' Feedback}

\author{\name Riccardo Poiani \email riccardo.poiani@polimi.it \\
       \addr DEIB, Politecnico di Milano \\
       Milano, Italy
       \AND
       \name Alberto Maria Metelli \email albertomaria.metelli@polimi.it \\
       \addr DEIB, Politecnico di Milano \\
       Milano, Italy
       \AND
       \name Marcello Restelli \email marcello.restelli@polimi.it \\
       \addr DEIB, Politecnico di Milano \\
       Milano, Italy}


\maketitle

\begin{abstract}
Stochastic multi-armed bandits are a sequential-decision-making framework, where, at each interaction step, the learner selects an arm and observes a stochastic reward. Within the context of best-arm identification (BAI) problems, the goal of the agent lies in finding the optimal arm, i.e., the one with the highest expected reward, as accurately and efficiently as possible. Nevertheless, the sequential interaction protocol of classical BAI problems, where the agent has complete control over the arm being pulled at each round, does not effectively model several decision-making problems of interest (e.g., off-policy learning, human feedback). For this reason, in this work, we propose a novel strict generalization of the classical BAI problem that we refer to as best-arm identification under mediators' feedback (BAI-MF). More specifically, we consider the scenario in which the learner has access to a set of \emph{mediators}, each of which selects the arms on the agent's behalf according to a stochastic and possibly \emph{unknown} policy. The mediator, then, communicates back to the agent the pulled arm together with the observed reward. In this setting, the agent's goal lies in sequentially choosing which mediator to query to identify with high probability the optimal arm while minimizing the identification time, i.e., the sample complexity. To this end, we first derive and analyze a statistical lower bound on the sample complexity specific to our general mediator feedback scenario. Then, we propose a sequential decision-making strategy for discovering the best arm; as our theory verifies, this algorithm matches the lower bound both almost surely and in expectation.\blfootnote{A concurrent study of this setting is presented in \citet{reddy2023best}. Our findings were derived independently and followed by different algorithmic choices and theoretical analyses.}
\end{abstract}


\section{Introduction}

Stochastic multi-armed bandits \citep{lattimorebandit2020} are a sequential decision-making framework where, during each interaction round, the learner selects an arm and observes a sample drawn from its reward distribution. Contrary to regret minimization problems, where the agent aims at maximizing the cumulative reward, in \emph{best-arm identification} (BAI) scenarios \citep{evendarpac2002}, the agent's primary focus lies in computing the arm with the highest expected reward (i.e., the optimal arm) as accurately and efficiently as possible. More specifically, in the \emph{fixed-confidence} setting, given a maximal risk parameter $\delta$, the agent's primary focus is on identifying, with probability at least $1-\delta$, the optimal arm with a minimum number of samples. Nevertheless, the sequential interaction protocol of classical BAI settings, in which the agent has complete control of the arm being pulled at each round (i.e., at each step, the agent chooses which arm to query), fails to adequately represent various decision-making problems that are of importance. In fact, in some relevant scenarios, the agent possesses only partial or no control over the arms being played. Consider, indeed, the following examples.
\begin{itemize}
\item \textbf{Off-Policy Learning.} Off-policy learning is a crucial aspect of decision-making theory that has gathered significant attention, especially within the Reinforcement Learning (RL) community \citep{sutton2018reinforcement}. Here, the agent continuously observes, at each round, actions sampled from a fixed behavioral policy, together with the corresponding rewards. The goal, here, consequently, lies in exploiting these off-policy interactions to identify the best arm with high probability. 
\item \textbf{Active Off-Policy Learning.} This scenario generalizes the off-policy setting previously presented. In this case, multiple behavioral policies are available to the agent. The learner can decide which behavioral policy to query to quickly identify the optimal arm. In practice, these behavioral policies can be, for instance, those of experts with the skill necessary to perform a subset of actions within the arm set. Another relevant example might arise in scenarios with human feedback \citep{li2019human}, where multiple humans can perform actions on the agent's behalf according to some private and personal policy.
\end{itemize}

As we can see, these scenarios cannot be properly modeled with the usual bandit interaction protocol as the agent has limited or no control on the arms being pulled during each interaction round. For this reason, in this work, we study a strict generalization of the classical BAI framework that circumvents the limits of complete controllability that is typical of bandit frameworks. To this end, we introduce the best-arm identification problem under \emph{mediators' feedback}, where the learner has access to a set of mediators, each of which will query arms on the agent’s behalf according to some stochastic, \emph{possibly unknown} and \emph{fixed} behavioral policy. The mediator will then communicate back to the agent which action it has played, together with the observed reward realization. In this setting, the agent’s goal lies in sequentially choosing which mediator to query to identify with high probability the optimal arm while minimizing the sample complexity. As one can verify, such formalism decouples the arms' pulls from the agent's choices, thus allowing to properly model all the scenarios depicted above. 

\section{Preliminaries and Backgrounds}

\subsection{Fixed-Confidence Best-Arm Identification}
In fixed-confidence best-arm identification (BAI) problems \citep{evendarpac2002}, the agent interacts with a set of $K$ probability distributions $\boldsymbol{\nu} = \left( \nu_1, \dots \nu_K \right)$ with respective means $\boldsymbol{\mu} = \left(\mu_1, \dots, \mu_K \right)$. For simplicity, we assume that there is a unique optimal arm, and, w.l.o.g., $\mu_1 > \mu_2 \ge \dots \ge \mu_K$. In the rest of this work, we consider distributions within the one-dimensional canonical exponential family \citep{cappe2013kullback}, which are directly parameterized by their mean.\footnote{The reader who is not familiar with the subject may consider Bernoullian or Gaussian distributions with known variance.} For this reason, with a little abuse of notation, we will often refer to the bandit model $\boldsymbol{\nu}$ using the means of its arms $\boldsymbol{\mu}$. We use the symbol $\mathcal{M}$ to denote this class of bandit models with unique optimal arms. Given two distributions $p, q \in \mathcal{M}$, we denote with $d(p, q)$ the KL divergence between $p$ and $q$.

We now proceed by formalizing the interaction scheme between the agent and the bandit model. At every interaction step $t \in \mathbb{N}$, the agent selects an arm $A_t \in [K]$ and receives a new and independent reward $X_t \sim \nu_{A_t}$. The procedure that defines how arms $A_t$ are selected is often referred to as \emph{sampling rule}. Given a maximal risk parameter $\delta \in (0, 1)$, the goal of the agent is to output the optimal arm $\hat{a}_{\tau_\delta} = \left\{ 1 \right\}$ with probability at least $1-\delta$, while minimizing the \emph{sample complexity} $\tau_\delta \in \mathbb{N}$. More formally, $\tau_\delta$ is a stopping time that controls the end of the data acquisition phase, after which a decision $\hat{a}_{\tau_\delta}$ is made. We refer to algorithms that satisfy $\mathbb{P} \left( \hat{a}_{\tau_\delta} \in \argmax_{a \in [K]} \mu_a \right) \le \delta$ as $\delta$-correct strategies. 

We now describe in detail the statistical complexity of fixed-confidence BAI problems \citep{garivieroptimal2016}. Given a bandit model $\boldsymbol{\mu} \in \mathcal{M}$, let $a^*(\boldsymbol{\mu}) = \argmax_{a \in [K]} \mu_a$. We introduce the set $\textrm{Alt}(\boldsymbol{\mu})$ as the set of problems where the optimal arm is different w.r.t. to $\boldsymbol{\mu}$, namely $\textrm{Alt}(\boldsymbol{\mu}) \coloneqq \left\{ \boldsymbol{\lambda} \in \mathcal{M} : a^*(\boldsymbol{\lambda}) \ne a^*(\boldsymbol{\mu}) \right\}$.  Let $\textrm{kl}(x, y) = x \log(x/y) +(1-x) \log((1-x) / (1-y))$. Then, for any $\delta$-correct algorithm it holds that $\mathbb{E}_{\boldsymbol{\mu}}\left[ \tau_\delta \right] \ge T^*(\boldsymbol{\mu}) \textrm{kl}(\delta, 1-\delta)$, where $T^*(\boldsymbol{\mu})^{-1}$ is given by: 
\begin{align}\label{eq:standard-char-time}
T^*(\boldsymbol{\mu})^{-1} = \sup_{\omega \in \Delta_K} \inf_{\boldsymbol{\lambda} \in \textrm{Alt}(\boldsymbol{\mu})} \left( \sum_{a=1}^K \omega_a d(\mu_a, \lambda_a) \right). 
\end{align}
We remark that, when $\delta \rightarrow 0$, $T^*(\boldsymbol{\mu})$ fully describes the statistical complexity of each problem $\boldsymbol{\mu}$. More specifically, it is possible to derive the following result: $\limsup_{\delta \rightarrow 0} \frac{\mathbb{E}_{\boldsymbol{\mu}} [\tau_\delta]}{\log(1/\delta)} \ge T^*(\boldsymbol{\mu})$. 
For this reason, $T^*(\boldsymbol{\mu})$ has played a crucial role in several BAI studies \citep[e.g.,][]{garivieroptimal2016,wang2021fast,tirinzoni2022elimination}. 
From Equation \eqref{eq:standard-char-time}, we can see that $T^*(\boldsymbol{\mu})^{-1}$ can be seen as a max-min game where the first player chooses a pull proportion among the different arms, and the second player chooses a hard-to-identify alternative problem where the optimal arm is different \citep{degenne2019non}.
In this sense, the unique maximizer of Equation \eqref{eq:standard-char-time}, which we denote as $\boldsymbol{\omega}^*(\boldsymbol{\mu}, \boldsymbol{\pi})$, can be interpreted as the optimal proportion with which arms should be queried in order to identify $a^*(\boldsymbol{\mu})$. Since solving Equation \eqref{eq:standard-char-time} requires access to quantities unknown to the learner, $\boldsymbol{\omega}^*(\boldsymbol{\mu}, \boldsymbol{\pi})$ often takes the name of \emph{oracle weights}.

\subsection{Best-Arm Identification under Mediators’ Feedback }
In this work, we study the following generalization of the best-arm identification problem. Given a bandit model $\boldsymbol{\nu}$ with $K$ arms, the learner cannot directly sample rewards from each arm $\nu_a$, but, instead it can query a set of $E$ mediators, each of which is described by a \emph{possibly unknown} and \emph{fixed} behavioral policy $\boldsymbol{\pi_e} \in \Delta_K$. More specifically, at each interaction step $t \in \mathbb{N}$, the agent will select a mediator $E_t \in [E]$, which, on the agent's behalf, will pull an arm $A_t \sim \boldsymbol{\pi_{E_t}}$ and will observe a reward $X_t \sim \nu_{A_t}$. The mediator $E_t$ will then communicate back to the agent both the action $A_t$ and observed reward $X_t$. For brevity, we adopt the symbol $\boldsymbol{\pi}$ as a shortcut for the set of mediators' policies $\left( \boldsymbol{\pi_e} \right)_{e=1}^E$. Given a maximal risk parameter $\delta$, the goal of the agent remains identifying with high-probability the optimal arm within $\boldsymbol{\mu}$ while minimizing the sample complexity $\tau_\delta$. To this end, we restict our study to the following scenarios.
\begin{ass}\label{ass:action-covering}
For any $a \in [K]$ there exists $e \in [E]$ such that $\pi_e(a) > 0$.
\end{ass}
Assumption \ref{ass:action-covering} states that the mediators' policies explores with positive probability each action $a \in [K]$. In other words, the agent should be able to gather information on each arm within the arm set. \footnote{We argue that this is a very mild requirement. Indeed, as we shall see in the appendix, Assumption \ref{ass:action-covering} is necessary for finite sample complexity results.}

To conclude, we notice that the proposed interaction protocol is a strict generalization w.r.t. to the usual BAI framework. Indeed, whenever (i) the mediators' policies are known, (ii) $E = K$ and, (iii) for all action $a \in [K]$, $\boldsymbol{\pi_{E_a}}$ is a Dirac distribution on action $a$, we recover the usual best-arm identification problem. In the rest of this document, we refer to this peculiar set of mediators' policies as $\boldsymbol{\bar{\pi}}$.

\section{On the Statistical Complexity}
This section discusses the intrinsic statistical complexity of the best-arm identification problems under mediators' feedback. More specifically, we provide and analyze a lower bound on the sample complexity that is necessary to identify the optimal arm with high-probability. 

\begin{restatable}{theorem}{est}\label{theo:lower-bound}
Let $\delta \in (0,1)$. For any $\delta$-correct strategy, any bandit model $\boldsymbol{\mu}$, and any set of mediators $\boldsymbol{\pi}$ it holds that $\mathbb{E}_{\boldsymbol{\mu}, \boldsymbol{\pi}} \left[ \tau_\delta \right] \ge \textup{kl}(\delta, 1-\delta) T^*(\boldsymbol{\mu}, \boldsymbol{\pi})$, 
where $T^*(\boldsymbol{\mu}, \boldsymbol{\pi})^{-1}$ is defined as:
\begin{align}\label{eq:lower-bound-time-ours}
\sup_{\boldsymbol{\omega} \in \Sigma_E} \inf_{\boldsymbol{\lambda} \in \textup{Alt}(\boldsymbol{\mu})} \left( \sum_{e=1}^E \omega_e \sum_{a=1}^K \pi_e(a) d(\mu_a, \lambda_a) \right).
\end{align}
\end{restatable}


Theorem \ref{theo:lower-bound} deserves some comments. First of all, as we can appreciate from Equation \eqref{eq:lower-bound-time-ours}, $T^*({\boldsymbol{\mu}, \boldsymbol{\pi}})^{-1}$ reports the typical max-min game that describes lower bounds for standard best-arm identification problems. More specifically, the max-player determines the proportion with which each mediator should be queried, while the min-player decides an alternative (and hard) alternative instance in which the optimal arm is modified. It has to be remarked that $T^*({\boldsymbol{\mu}, \boldsymbol{\pi}})^{-1}$  ,and, consequently, the oracle weights $\boldsymbol{\omega}^*(\boldsymbol{\mu}, \boldsymbol{\pi})$, directly depends on the set of mediators' policies $\boldsymbol{\pi}$. In other words, $\boldsymbol{\pi}$ plays a crucial role in the statistical complexity of the problem. To further investigate this dipendency, let us introduce some additional notation. Given $\boldsymbol{\omega} \in \Sigma_E$, we  define $\boldsymbol{\tilde{\pi}}(\boldsymbol{\omega}) \in \Sigma_K$, where ${\tilde{\pi}}_a(\boldsymbol{\omega}) = \sum_{e=1}^E \omega_e \pi_e(a)$ denotes the probability of playing an arm $a$ when sampling mediators according to $\boldsymbol{\omega}$. Then, let $\widetilde{\Sigma}_K \subseteq \Sigma_K$ be the set of all the possible $\boldsymbol{\tilde{\pi}}$ that can be obtained starting from any $\boldsymbol{\omega} \in \Sigma_E$. Given this notation, it is possible to rewrite $T^*({\boldsymbol{\mu}, \boldsymbol{\pi}})^{-1}$ as:
\begin{align}\label{eq:lower-bound-time-ours-2}
\sup_{\boldsymbol{\tilde{\pi}} \in \widetilde{\Sigma}_K} \inf_{\boldsymbol{\lambda} \in \textrm{Alt}(\boldsymbol{\mu})} \sum_{a=1}^K \tilde{\pi}_a d(\mu_a, \lambda_a).
\end{align}
At this point, we notice that Equation \eqref{eq:lower-bound-time-ours-2} shares significant similarities with the definition of $T^*(\boldsymbol{\mu})^{-1}$ for classical BAI problems; i.e., Equation \eqref{eq:standard-char-time}. The only difference, indeed, stands in the fact that, under mediators' feedback, the max-player can only act on the restricted set $\widetilde{\Sigma}_K$ rather than the entire simplex $\Sigma_K$. In this sense, the max-min game is between the proportion of arm pulls that is \emph{possible} to play according to the mediators $\boldsymbol{\pi}$, and the alternative hard instance. In the rest of this document, we denote maximizers of Equation \eqref{eq:lower-bound-time-ours-2} with $\boldsymbol{\tilde{\pi}}^*(\boldsymbol{\mu}, \boldsymbol{\pi})$. Given this interpreation of Theorem \ref{theo:lower-bound} we now proceed by further investigating the comparison with classical BAI problems.\footnote{Further analysis on the statistical complexity are provided in the appendix.}

\subsection{Comparison with classical BAI}
First of all, it is worth noting that Theorem \ref{theo:lower-bound} effectively generalizes existing statistical complexity results of the typical BAI problem, thus offering a broader perspective. Indeed, whenever the set of mediators' policies is equal to $\boldsymbol{\bar{\pi}}$, Theorem \ref{theo:lower-bound} directly reduces to the usual BAI lower bound. In other words, $T^*(\boldsymbol{\mu}, \boldsymbol{\bar{\pi}})^{-1}$ is exactly $T^*(\boldsymbol{\mu})^{-1}$. Furthermore, for a general set of mediators $\boldsymbol{\pi}$, it is possible to derive the following result. 

\begin{restatable}{proposition}{comparison}\label{prop:comparison}
For any bandit model $\boldsymbol{\mu}$ and mediators' policies $\boldsymbol{\pi}$ it holds that:
\begin{align}\label{eq:prop-one}
T^*(\boldsymbol{\mu}, \boldsymbol{\pi})^{-1} \le T^*(\boldsymbol{\mu}, \boldsymbol{\bar{\pi}})^{-1}.
\end{align}
Furthermore, $T^*(\boldsymbol{\mu}, \boldsymbol{\pi})^{-1} < T^*(\boldsymbol{\mu}, \boldsymbol{\bar{\pi}})^{-1}$ holds if and only if $\boldsymbol{\omega^*}(\boldsymbol{\mu}, \boldsymbol{\bar{\pi}}) \notin \widetilde{\Sigma}_K$.
\end{restatable}

From Equation \ref{eq:prop-one}, we can see that the mediators' feedback problem is always at least as difficult as the classical BAI setting. From an intuitive perspective, this result is expected. Indeed, from Equation \eqref{eq:lower-bound-time-ours-2}, we know that the only difference between $T^*(\boldsymbol{\mu}, \boldsymbol{\pi})^{-1}$ and $T^*(\boldsymbol{\mu}, \boldsymbol{\bar{\pi}})^{-1}$ lies in the definition of $\widetilde{\Sigma}_K$, that, as previously discussed,  encodes the partial controllability on the arm space that is introduced by the mediators $\boldsymbol{\pi}$. Furthermore, Proposition \ref{prop:comparison} fully characterizes the set of instances in which the mediators' feedback introduces additional challenges in identifying the optimal arm. More precisely, the lower bound of Theorem \ref{theo:lower-bound} separates from the one of classical BAI whenever the max-player cannot pull, in expectation, arms according to the proportion $\boldsymbol{\omega^*}(\boldsymbol{\mu}, \boldsymbol{\bar{\pi}})$ that results from the lower bound of the classical BAI problem.

\section{Track and Stop under Mediators' feedback}\label{sec:tas}
In this section, we continue by providing our algorithm for the best-arm identification problem under mediators feedback. Here, we focus on the case in which the mediators policies $\boldsymbol{\pi}$ are known to the learner.\footnote{We refer the reader to the appendix for the case in which $\boldsymbol{\pi}$ is unknown. Nevertheless, we remark that, with a slight modification to the algorithm, it is possible to obtain identical theoretical results.} As algorithm, we cast the Track and Stop (TaS) framework \citep{garivieroptimal2016} to our interaction setting in the following way.\footnote{We report a short description of the classical TaS algorithm in the appendix; for further details see \cite{garivieroptimal2016}.}

\paragraph{Sampling Rule} As a sampling rule, we adopt C-tracking \citep{garivieroptimal2016} of the oracle mediator proportions $\boldsymbol{\omega}^*(\boldsymbol{\mu}, \boldsymbol{\pi})$. More formally, let $\boldsymbol{\hat{\mu}}(t)$ be the vector of estimates of the mean of each arm at time $t$. We then compute any maximizers of the empirical version of Equation \eqref{eq:lower-bound-time-ours} (i.e., $\boldsymbol{\mu}$ is replaced with $\boldsymbol{\hat{\mu}}$), and we $L^{\infty}$ project it onto $\Sigma_E^{\epsilon_t} = \left\{ \boldsymbol{\omega} \in \Sigma_E : \forall i \textrm{ }\omega_i \ge \epsilon_t \right\}$, where $\epsilon_t$ is given by $\epsilon_t = (E^2 + t)^{-1/2}/2$. We notice that, in the original version of TaS, C-Tracking was applied to track optimal proportions between arms. 
Our algorithmic choice (i.e., tracking mediator proportions) is a direct consequence of the fact that we cannot directly track arm proportions (e.g., $\boldsymbol{\tilde{\pi}}^*(\boldsymbol{\mu}, \boldsymbol{\pi})$), but, instead, the learner can only decide which mediator will be queried at time $t$. 


\paragraph{Stopping Rule} Since the goal lies in identifying the optimal arm $a \in [K]$, we stick to the successful Generalized Likelihood Ratio (GLR) statistic to decide when enough information has been gathered to confidently reccomend which arm has the highest mean \citep{garivieroptimal2016}. 


\paragraph{The reccomendation} For the same reasons of the stopping rule, we rely on the reccomendation rule of \citet{garivieroptimal2016}. Namely, our algorithm reccomend the arm with the highest empirical mean $\hat{a}_{\tau_\delta} = \argmax_{a \in [K]} \hat{\mu}_a(\tau_\delta)$. 


\subsection{Theoretical Results}
At this point, we are ready to present our theoretical analysis on the performance of our algorithm. We begin by providing the following almost surely convergence result. 

\begin{restatable}{theorem}{highprob}\label{theo:high-prob}
Consider any $\boldsymbol{\mu} \in \mathcal{M}$ and any $\boldsymbol{\pi}$ such that Assumption \ref{ass:action-covering} is satisfied. Let $\alpha \in (1, e/2]$. It holds that:
\begin{align}\label{eq:high-prob}
\mathbb{P}_{\boldsymbol{\mu}, \boldsymbol{\pi}} \left( \limsup_{\delta \rightarrow 0 } \frac{\tau_\delta}{\log\left( 1/\delta \right)} \le \alpha T^*(\boldsymbol{\mu}, \boldsymbol{\pi})\right) = 1.
\end{align}
\end{restatable}

Similarly, it is possible to derive a result that directly controls the expectation of the stopping time $\tau_\delta$. More specifically, we prove the following result.

\begin{restatable}{theorem}{expect}\label{theo:expect}
Consider any $\boldsymbol{\mu} \in \mathcal{M}$ and any $\boldsymbol{\pi}$ such that Assumption \ref{ass:action-covering} is satisfied. Let $\alpha \in (1, e/2]$. It holds that:
\begin{align}\label{eq:expect}
\limsup_{\delta \rightarrow 0 } \frac{\mathbb{E}_{\boldsymbol{\mu}, \boldsymbol{\pi}}[\tau_\delta]}{\log(1/\delta)} \le \alpha T^*(\boldsymbol{\mu}, \boldsymbol{\pi}).
\end{align}
\end{restatable}

In other words, Theorem \ref{theo:expect} shows that in the asymptotic regime of $\delta \rightarrow 0$, our algorithm matches the lower bound presented in Theorem \ref{theo:lower-bound}.


\acks{This paper is supported by PNRR-PE-AI FAIR project funded by the NextGeneration EU program.}


\newpage

\appendix

\vskip 0.2in
\bibliography{sample}

\begin{thebibliography}{37}
\providecommand{\natexlab}[1]{#1}
\providecommand{\url}[1]{\texttt{#1}}
\expandafter\ifx\csname urlstyle\endcsname\relax
  \providecommand{\doi}[1]{doi: #1}\else
  \providecommand{\doi}{doi: \begingroup \urlstyle{rm}\Url}\fi

\bibitem[Agrawal et~al.(2020)Agrawal, Juneja, and Glynn]{agrawal2020optimal}
Shubhada Agrawal, Sandeep Juneja, and Peter Glynn.
\newblock Optimal -correct best-arm selection for heavy-tailed distributions.
\newblock In \emph{Algorithmic Learning Theory}, pages 61--110. PMLR, 2020.

\bibitem[Altschuler et~al.(2019)Altschuler, Brunel, and
  Malek]{altschuler2019best}
Jason~M Altschuler, Victor-Emmanuel Brunel, and Alan Malek.
\newblock Best arm identification for contaminated bandits.
\newblock \emph{J. Mach. Learn. Res.}, 20\penalty0 (91):\penalty0 1--39, 2019.

\bibitem[Aubin and Frankowska(2009)]{aubin2009set}
Jean-Pierre Aubin and H{\'e}l{\`e}ne Frankowska.
\newblock \emph{Set-valued analysis}.
\newblock Springer Science \& Business Media, 2009.

\bibitem[Audibert et~al.(2010)Audibert, Bubeck, and Munos]{audibertbest2010}
Jean-Yves Audibert, Sébastien Bubeck, and Rémi Munos.
\newblock Best arm identification in multi-armed bandits.
\newblock In \emph{{COLT}}, pages 41--53. Citeseer, 2010.

\bibitem[Bacchiocchi et~al.(2023)Bacchiocchi, Stradi, Papini, Metelli, and
  Gatti]{bacchiocchi2023online}
Francesco Bacchiocchi, Francesco~Emanuele Stradi, Matteo Papini, Alberto~Maria
  Metelli, and Nicola Gatti.
\newblock Online adversarial mdps with off-policy feedback and known
  transitions.
\newblock In \emph{Sixteenth European Workshop on Reinforcement Learning},
  2023.

\bibitem[Bubeck et~al.(2009)Bubeck, Munos, and Stoltz]{bubeckpure2009}
Sébastien Bubeck, Rémi Munos, and Gilles Stoltz.
\newblock Pure exploration in multi-armed bandits problems.
\newblock In \emph{International conference on {Algorithmic} learning theory},
  pages 23--37. Springer, 2009.

\bibitem[Capp{\'e} et~al.(2013)Capp{\'e}, Garivier, Maillard, Munos, and
  Stoltz]{cappe2013kullback}
Olivier Capp{\'e}, Aur{\'e}lien Garivier, Odalric-Ambrym Maillard, R{\'e}mi
  Munos, and Gilles Stoltz.
\newblock Kullback-leibler upper confidence bounds for optimal sequential
  allocation.
\newblock \emph{The Annals of Statistics}, pages 1516--1541, 2013.

\bibitem[Dann et~al.(2017)Dann, Lattimore, and Brunskill]{dann2017unifying}
Christoph Dann, Tor Lattimore, and Emma Brunskill.
\newblock Unifying pac and regret: Uniform pac bounds for episodic
  reinforcement learning.
\newblock \emph{Advances in Neural Information Processing Systems}, 30, 2017.

\bibitem[Degenne and Koolen(2019)]{degenne2019pure}
R{\'e}my Degenne and Wouter~M Koolen.
\newblock Pure exploration with multiple correct answers.
\newblock \emph{Advances in Neural Information Processing Systems}, 32, 2019.

\bibitem[Degenne et~al.(2019)Degenne, Koolen, and M{\'e}nard]{degenne2019non}
R{\'e}my Degenne, Wouter~M Koolen, and Pierre M{\'e}nard.
\newblock Non-asymptotic pure exploration by solving games.
\newblock \emph{Advances in Neural Information Processing Systems}, 32, 2019.

\bibitem[Degenne et~al.(2020)Degenne, M{\'e}nard, Shang, and
  Valko]{degenne2020gamification}
R{\'e}my Degenne, Pierre M{\'e}nard, Xuedong Shang, and Michal Valko.
\newblock Gamification of pure exploration for linear bandits.
\newblock In \emph{International Conference on Machine Learning}, pages
  2432--2442. PMLR, 2020.

\bibitem[Eldowa et~al.(2023)Eldowa, Esposito, Cesari, and
  Cesa-Bianchi]{eldowa2023minimax}
Khaled Eldowa, Emmanuel Esposito, Tommaso Cesari, and Nicol{\`o} Cesa-Bianchi.
\newblock On the minimax regret for online learning with feedback graphs.
\newblock \emph{arXiv preprint arXiv:2305.15383}, 2023.

\bibitem[Even-Dar et~al.(2002)Even-Dar, Mannor, and Mansour]{evendarpac2002}
Eyal Even-Dar, Shie Mannor, and Yishay Mansour.
\newblock {PAC} bounds for multi-armed bandit and {Markov} decision processes.
\newblock In \emph{International {Conference} on {Computational} {Learning}
  {Theory}}, pages 255--270. Springer, 2002.

\bibitem[Feng et~al.(2022)Feng, Caldentey, and Ryan]{feng2022robust}
Yifan Feng, Ren{\'e} Caldentey, and Christopher~Thomas Ryan.
\newblock Robust learning of consumer preferences.
\newblock \emph{Operations Research}, 70\penalty0 (2):\penalty0 918--962, 2022.

\bibitem[Gabbianelli et~al.(2023)Gabbianelli, Neu, and
  Papini]{gabbianelli2023online}
Germano Gabbianelli, Gergely Neu, and Matteo Papini.
\newblock Online learning with off-policy feedback.
\newblock In \emph{International Conference on Algorithmic Learning Theory},
  pages 620--641. PMLR, 2023.

\bibitem[Garivier et~al.(2017)Garivier, M{\'e}nard, Rossi, and
  Menard]{garivier2017thresholding}
Aur{\'e}lien Garivier, Pierre M{\'e}nard, Laurent Rossi, and Pierre Menard.
\newblock Thresholding bandit for dose-ranging: The impact of monotonicity.
\newblock \emph{arXiv preprint arXiv:1711.04454}, 2017.

\bibitem[Garivier and Kaufmann(2016)]{garivieroptimal2016}
Aurélien Garivier and Emilie Kaufmann.
\newblock Optimal best arm identification with fixed confidence.
\newblock In \emph{Conference on {Learning} {Theory}}, pages 998--1027. PMLR,
  2016.

\bibitem[Jamieson and Nowak(2014)]{jamiesonbestarm2014}
Kevin Jamieson and Robert Nowak.
\newblock Best-arm identification algorithms for multi-armed bandits in the
  fixed confidence setting.
\newblock In \emph{2014 48th {Annual} {Conference} on {Information} {Sciences}
  and {Systems} ({CISS})}, pages 1--6. IEEE, 2014.

\bibitem[Jamieson et~al.(2014)Jamieson, Malloy, Nowak, and
  Bubeck]{jamiesonlilucb2014}
Kevin Jamieson, Matthew Malloy, Robert Nowak, and Sébastien Bubeck.
\newblock lil’ucb: {An} optimal exploration algorithm for multi-armed
  bandits.
\newblock In \emph{Conference on {Learning} {Theory}}, pages 423--439. PMLR,
  2014.

\bibitem[Jedra and Proutiere(2020)]{jedra2020optimal}
Yassir Jedra and Alexandre Proutiere.
\newblock Optimal best-arm identification in linear bandits.
\newblock \emph{Advances in Neural Information Processing Systems},
  33:\penalty0 10007--10017, 2020.

\bibitem[Karnin et~al.(2013)Karnin, Koren, and Somekh]{karninalmost2013}
Zohar Karnin, Tomer Koren, and Oren Somekh.
\newblock Almost optimal exploration in multi-armed bandits.
\newblock In \emph{International {Conference} on {Machine} {Learning}}, pages
  1238--1246. PMLR, 2013.

\bibitem[Kaufmann et~al.(2016)Kaufmann, Cappé, and
  Garivier]{kaufmanncomplexity2016}
Emilie Kaufmann, Olivier Cappé, and Aurélien Garivier.
\newblock On the complexity of best-arm identification in multi-armed bandit
  models.
\newblock \emph{The Journal of Machine Learning Research}, 17\penalty0
  (1):\penalty0 1--42, 2016.
\newblock Publisher: JMLR. org.

\bibitem[Koc{\'a}k and Garivier(2020)]{kocak2020best}
Tom{\'a}{\v{s}} Koc{\'a}k and Aur{\'e}lien Garivier.
\newblock Best arm identification in spectral bandits.
\newblock \emph{arXiv preprint arXiv:2005.09841}, 2020.

\bibitem[Lattimore and Szepesvári(2020)]{lattimorebandit2020}
Tor Lattimore and Csaba Szepesvári.
\newblock \emph{Bandit algorithms}.
\newblock Cambridge University Press, 2020.

\bibitem[Levine et~al.(2020)Levine, Kumar, Tucker, and Fu]{levine2020offline}
Sergey Levine, Aviral Kumar, George Tucker, and Justin Fu.
\newblock Offline reinforcement learning: Tutorial, review, and perspectives on
  open problems.
\newblock \emph{arXiv preprint arXiv:2005.01643}, 2020.

\bibitem[Li et~al.(2019)Li, Gomez, Nakamura, and He]{li2019human}
Guangliang Li, Randy Gomez, Keisuke Nakamura, and Bo~He.
\newblock Human-centered reinforcement learning: A survey.
\newblock \emph{IEEE Transactions on Human-Machine Systems}, 49\penalty0
  (4):\penalty0 337--349, 2019.

\bibitem[Metelli et~al.(2021)Metelli, Papini, D'Oro, and
  Restelli]{metelli2021policy}
Alberto~Maria Metelli, Matteo Papini, Pierluca D'Oro, and Marcello Restelli.
\newblock Policy optimization as online learning with mediator feedback.
\newblock In \emph{Proceedings of the AAAI Conference on Artificial
  Intelligence}, volume~35, pages 8958--8966, 2021.

\bibitem[Moulos(2019)]{moulos2019optimal}
Vrettos Moulos.
\newblock Optimal best markovian arm identification with fixed confidence.
\newblock \emph{Advances in Neural Information Processing Systems}, 32, 2019.

\bibitem[Mukherjee et~al.(2021)Mukherjee, Tajer, Chen, and
  Das]{mukherjee2021mean}
Arpan Mukherjee, Ali Tajer, Pin-Yu Chen, and Payel Das.
\newblock Mean-based best arm identification in stochastic bandits under reward
  contamination.
\newblock \emph{Advances in Neural Information Processing Systems},
  34:\penalty0 9651--9662, 2021.

\bibitem[Papini et~al.(2019)Papini, Metelli, Lupo, and
  Restelli]{papini2019optimistic}
Matteo Papini, Alberto~Maria Metelli, Lorenzo Lupo, and Marcello Restelli.
\newblock Optimistic policy optimization via multiple importance sampling.
\newblock In \emph{International Conference on Machine Learning}, pages
  4989--4999. PMLR, 2019.

\bibitem[Reddy et~al.(2023)Reddy, Karthik, Karamchandani, and
  Nair]{reddy2023best}
Kota~Srinivas Reddy, P.~N. Karthik, Nikhil Karamchandani, and Jayakrishnan
  Nair.
\newblock Best arm identification in bandits with limited precision sampling,
  2023.

\bibitem[Russac et~al.(2021)Russac, Katsimerou, Bohle, Capp{\'e}, Garivier, and
  Koolen]{russac2021b}
Yoan Russac, Christina Katsimerou, Dennis Bohle, Olivier Capp{\'e},
  Aur{\'e}lien Garivier, and Wouter~M Koolen.
\newblock A/b/n testing with control in the presence of subpopulations.
\newblock \emph{Advances in Neural Information Processing Systems},
  34:\penalty0 25100--25110, 2021.

\bibitem[Sen et~al.(2018)Sen, Shanmugam, and Shakkottai]{sen2018contextual}
Rajat Sen, Karthikeyan Shanmugam, and Sanjay Shakkottai.
\newblock Contextual bandits with stochastic experts.
\newblock In \emph{International Conference on Artificial Intelligence and
  Statistics}, pages 852--861. PMLR, 2018.

\bibitem[Sutton and Barto(2018)]{sutton2018reinforcement}
Richard~S Sutton and Andrew~G Barto.
\newblock \emph{Reinforcement learning: An introduction}.
\newblock MIT press, 2018.

\bibitem[Tirinzoni and Degenne(2022)]{tirinzoni2022elimination}
Andrea Tirinzoni and R{\'e}my Degenne.
\newblock On elimination strategies for bandit fixed-confidence identification.
\newblock \emph{Advances in Neural Information Processing Systems},
  35:\penalty0 18586--18598, 2022.

\bibitem[Wang et~al.(2021)Wang, Tzeng, and Proutiere]{wang2021fast}
Po-An Wang, Ruo-Chun Tzeng, and Alexandre Proutiere.
\newblock Fast pure exploration via frank-wolfe.
\newblock \emph{Advances in Neural Information Processing Systems},
  34:\penalty0 5810--5821, 2021.

\bibitem[Yang and Feng(2023)]{yang2023nested}
Junwen Yang and Yifan Feng.
\newblock Nested elimination: A simple algorithm for best-item identification
  from choice-based feedback.
\newblock \emph{International Conference on Machine Learning}, 2023.

\end{thebibliography}

\newpage

\section*{Appendix A: Related Works}\label{app:rel-works}
In this section, we provide in-depth discussion with works that are related to ours.

\paragraph{Best-Arm Identification}
Since the seminal work of \citet{evendarpac2002}, the fixed-confidence BAI setting has gathered increasing attention within the community. In particular, considerable efforts have been dedicated to refining algorithms and statistical lower-bounds, all with the ultimate objective of constructing optimal identification strategies \citep[e.g.,][]{bubeckpure2009,audibertbest2010,karninalmost2013,jamiesonlilucb2014,jamiesonbestarm2014,kaufmanncomplexity2016}. In this context, and of particular relevance for our work, \citet{garivieroptimal2016} have proposed the celebrated Track and Stop (TaS) algorithm, which attains \emph{optimal} statistical complexity in the asymptotic regime, i.e., $\delta \rightarrow 0$. Building upon this work, numerous studies have been conducted to propose improvements and generalizations upon the TaS algorithm \citep[e.g.,][]{degenne2019pure,wang2021fast,degenne2020gamification,tirinzoni2022elimination}.
 
\paragraph{Structured Best-Arm Identification}
One of the key factors that contributed to the success of TaS is its ability to emerge as a versatile framework that can be meticulously adapted to several variants of the BAI problem, such as linear \citep{jedra2020optimal} and spectral bandits \citep{kocak2020best}, multiple answers problems \citep{degenne2019pure}, and many others \citep[e.g.,][]{garivier2017thresholding,moulos2019optimal,agrawal2020optimal}. Among problems with additional structure, our work is related to BAI under choice-based feedback \citep{feng2022robust,yang2023nested}, where a company sequentially shows sets of items to a population of customers and collects their choices. The objective is to identify the most preferred item with the least number of samples and with high-probability. Another relevant work is \citet{russac2021b}, where the authors study the BAI problem in the presence of sub-populations. In more precise terms, the authors make the assumption that a population can be divided into distinct and similar subgroups. During each time step, one of these subgroups is sampled and an action (i.e., arm) is chosen . The observed outcome is a random draw from the selected arm, considering the characteristics of the current subgroup. To evaluate the effectiveness of each arm, a weighted average of its subpopulation means is used. Finally, our feedback structure is also related to BAI in contaminated bandits \citep{altschuler2019best,mukherjee2021mean}, where each arm pull has a probability $\epsilon$ of generating a sample from an arbitrary distribution, rather than the true one. Nevertheless, we remark that none of these settings can be mapped to the mediators' feedback one and viceversa.

\paragraph{Mediators' Feedback} The mediator feedback terminology was introduced by \citet{metelli2021policy} in the context of Policy Optimization (PO) in RL. Similar to the previous studies of \citet{papini2019optimistic}, the authors deal with the PO problem as a bandit where each policy in a given set is mapped to a distinct arm, whose reward is given by the usual cumulative RL return. Notice that, in this setting, the ability to perform actions in the environment is mediated by the policy set of the agent. For this reason, in our work, we adopt their terminology to disentangle the arms' pull from the agent’s choices. Among this line of works, we notice that, recently, a variant of this problem has also been studied in the context of non-stochastic bandits with expert advice \citep{eldowa2023minimax}. Here, during each round, the learner selects an expert that will perform an action on the agent's behalf according to some fixed distribution. Similar ideas have also been investigated in \citet{sen2018contextual} for regret minimization in contextual bandits. More specifically, they assume access to a class of stochastic experts, where each expert is a conditional distribution over the arms given the context. Compared to our work both \citet{sen2018contextual} and \citet{eldowa2023minimax} consider the problem of minimizing the regret against the best expert.

\paragraph{Other Related Works} Off-policy learning plays a vital role in decision-making theory and has garnered considerable interest, particularly in RL \citep{sutton2018reinforcement}. In particular, the off-policy feedback has received extensive research in the offline RL literature \citep{levine2020offline}, where the agent lacks the ability to directly interact with the environment and is instead limited to utilizing a fixed dataset gathered by possible multiple and unknown behaviorial policies. Finally, related to our work, \citet{gabbianelli2023online} have studied regret minimization in adversarial bandits with off-policy feedback. More specifically, the authors assume that the learner cannot directly observe its rewards, but instead sees the ones obtained by a behavioral and unknown policy that runs in parallel. Similar ideas have also been extended to the MDP setting with known transitions in \citet{bacchiocchi2023online}.

\section*{Appendix B: Further Details on Track and Stop}
The seminal work of \citet{garivieroptimal2016} has presented the Track and Stop (TaS) algorithm, which is the first asymptotically optimal approach for the fixed confidence BAI scenario; i.e., when $\delta \rightarrow 0$, it guarantees to stop with sample complexity that matches the lower bound. The core idea behind TaS lies in solving an empirical version of Equation \eqref{eq:standard-char-time} to estimate the optimal oracle weights. Then, in order to match the optimal proportions $\bm{\omega}^*(\bm{\mu}, \bm{\pi})$ (which guarantees to achieve optimality), the sampling rule will allocate samples by \emph{tracking} this empirical estimation. We remark that this is combined with a forced exploration sampling strategy that ensures that the estimate of the mean of each arm, and consequently the estimate of $\bm{\omega}^*(\bm{\mu}, \bm{\pi})$, is sufficiently accurate. Lastly, as a stopping criterion, TS employs the Generalized Likelihood Ratio (GLR) statistic to determine if enough information has been collected to infer, with a risk not exceeding $\delta$, whether the mean of one arm is greater than that of all the others. 
More specifically, the algorithm stops whenever the following condition is verified:
\begin{align}\label{eq:stopping}
Z(t) \coloneqq t \min_{\bm{\lambda} \in \textrm{Alt}(\bm{\hat{\mu}})} \sum_{a=1}^K \frac{N_a(t)}{t} d(\hat{\mu}_a(t), \lambda_a) \ge \beta(t, \delta),
\end{align}
where $N_a(t)$ denotes the number of pulls to arm $a$ at time $t$, and $\beta(t,\delta)$ represents an exploration rate that is commonly set to $\log\left( \frac{Ct^{\alpha}}{\delta}\right)$, for some $\alpha > 1$ and appropriate constanct $C$.\footnote{We refer the reader to the original work of \citet{garivieroptimal2016} for a formal exposition of the GLR statistic and its use within pure exploration problems.}

\section*{Appendix C: Further Analysis on the Statistical Complexity}
In this section, we provide additional analysis on the statistical complexity of best arm identification under mediators' feedback. All statements will be formally proven in Appendix E.

\subsection*{On the Action Covering Assumption}
We now continue with a formal justification behind the action covering assumption, i.e., Assumption \ref{ass:action-covering}. To this end, we analyze the behavior of Theorem \ref{theo:lower-bound} under the peculiar single mediator setting, that is $E=1$. More specifically, let us focus on the case in which there are two different actions, $a_1$ and $a_2$, associated to Gaussian reward distributions with unitary variance and means $\mu_{a_1} > \mu_{a_2}$. In this case, $T^*(\boldsymbol{\mu}, \boldsymbol{\pi})^{-1}$ reduces to:
\begin{align}\label{eq:justification}
\frac{1}{2} \frac{\pi_e(a_1) \pi_e(a_2)}{\pi_e(a_1) + \pi_e(a_2)} \Delta^2,
\end{align}
where $\Delta = \mu_{a_1} - \mu_{a_2}$. In this context, it is easy to see that, as soon as $\pi_e(a) \rightarrow 1$ for any of the two actions, $T^*(\boldsymbol{\mu}, \boldsymbol{\pi})^{-1}$ tends to $0$, and, consequently, $\mathbb{E}_{\boldsymbol{\mu}, \boldsymbol{\pi}}[\tau_\delta] \rightarrow +\infty$. In this sense, we can appreciate as Assumption \ref{ass:action-covering} turns out being a necessary assumption for finite sample complexity result. This should come as no surprise: if we cannot observe any realization from a certain arm $a \in [K]$, we are unable to conclude whether $a$ is optimal or not.

\subsection*{Off-Policy Learning}
Given the significant importance of off-policy learning within the sequential decision-making community, we now provide additional details on the lower bound for the case in which $E=1$. We notice, indeed, that whenever $E=1$, our setting reduces to the off-policy best-arm identification problem, where the learner continuously observes actions and rewards from another agent (i.e, the mediator). In this case, assuming for the sake of exposition Gaussian distributions with unitary variance, $T^*(\boldsymbol{\mu}, \boldsymbol{\pi})^{-1}$ can be rewritten as:
\begin{align}\label{eq:lower-bound-single-expert}
T^*(\boldsymbol{\mu}, \boldsymbol{\pi})^{-1} = \min_{a \ne 1} \frac{1}{2} \frac{\pi_e(1) \pi_e(a)}{\pi_e(1) + \pi_e(a)} \Delta^2_a,
\end{align}
where $\Delta_a = \mu_1 - \mu_a$. Equation \eqref{eq:lower-bound-single-expert} expresses the lower bound \emph{only} in term of the most difficult to identify alternative arm (i.e., the minimum over the different sub-optimal actions). Furthermore, contrary to what usually happens in classical BAI problems, this difficult to identify alternative arm is not the one with the smallest gap $\Delta_a$, but there is a trade-off between $\Delta_a$ and how easy it is to observe the mediator playing action $a$, namely $\pi_e(a)$.

\subsection*{On $\boldsymbol{\omega}^*(\boldsymbol{\mu}, \boldsymbol{\pi})$ and $\boldsymbol{\tilde{\pi}}^*(\boldsymbol{\mu}, \boldsymbol{\pi})$}
Finally, we conclude with some more technical considerations on $\boldsymbol{\omega}^*(\boldsymbol{\mu}, \boldsymbol{\pi})$ and $\boldsymbol{\tilde{\pi}}^*(\boldsymbol{\mu}, \boldsymbol{\pi})$. It has to be noticed that, compared to standard BAI problems, $\boldsymbol{\omega}^*(\boldsymbol{\mu}, \boldsymbol{\pi})$ and $\boldsymbol{\tilde{\pi}}^*(\boldsymbol{\mu}, \boldsymbol{\pi})$ are, in general, not unique. In other words, the mappings $\boldsymbol{\omega}^*(\boldsymbol{\mu}, \boldsymbol{\pi})$ and $\boldsymbol{\tilde{\pi}}^*(\boldsymbol{\mu}, \boldsymbol{\pi})$ are set-valued. \footnote{As a simple example, it is sufficient to consider the case in which linearly-dependent policies are present in $\bm{\pi}$.} As shown by previous works \citep{degenne2019pure}, this sort of feature introduces significant challenges within the algorithmic design/analysis of, e.g., Track and Stop inspired algorithms. The following proposition shows significant and technical relevant properties that help overcoming these challenges. 

\begin{restatable}{proposition}{upperhemi}\label{prop:omega-star-properties}
The sets $\boldsymbol{\omega}^*(\boldsymbol{\mu}, \boldsymbol{\pi})$ and $\boldsymbol{\tilde{\pi}}^*(\boldsymbol{\mu}, \boldsymbol{\pi})$ are convex. Furthermore, the mappings $(\boldsymbol{\mu}, \boldsymbol{\pi}) \rightarrow \boldsymbol{\omega}^*(\boldsymbol{\mu}, \boldsymbol{\pi})$ and $(\boldsymbol{\mu}, \boldsymbol{\pi}) \rightarrow \boldsymbol{\tilde{\pi}}^*(\boldsymbol{\mu}, \boldsymbol{\pi})$ are upper hemicontinuous.
\end{restatable}

The unfamiliar reader might think of upper hemicontinuity as a generalization of the continuity property for set-valued mappings \citep{aubin2009set}. \footnote{Further technical details on this point are deferred to the appendix.} As our analysis will reveal, Proposition \ref{prop:omega-star-properties} will play a crucial role for the analysis of our algorithmic solution.


\section*{Appendix D: Unknown Mediators' policies}

In this section, we extend our results to the case in which the agent does not know $\bm{\pi}$, but instead it has learn it directly from data. All theoretical results are formally proven in the Appendix E.

Before detailing our theoretical findings, we notice that Theorem \ref{theo:lower-bound} still represents a valid lower bound for this more intricated setting. For this reason, one might be tempted to extend our algorithm to track the optimal mediators proportions $\bm{\omega}^*(\bm{\mu}, \bm{\pi})$ to the case in which the set $\bm{\pi}$ is unknown to the learner. To this end, let $\bm{\hat{\pi}}(t)$ be the matrix containing empirical estimates for each mediator policy $\bm{\pi_e}$. Then, it is sufficient to modify the C-tracking sampling rule presented in Section \ref{sec:tas} by computing any maximizer of the empirical version of Equation \eqref{eq:lower-bound-time-ours} where \emph{both} $\bm{\mu}$ and $\bm{\pi}$ are replaced with $\bm{\hat{\mu}}(t)$ and $\bm{\hat{\pi}}(t)$ respectively. \footnote{We notice that, in the previous section, $\bm{\mu}$ was replaced with $\bm{\hat{\mu}}(t)$ while $\bm{\pi}$ was used directly since it was available to the learner.} As we shall now see, this simple modification allows us to derive results that are equivalent to Theorem \ref{theo:high-prob} and \ref{theo:expect}. More specifically, we begin by showing the following almost surely convergence result. 

\begin{theorem}\label{theo:high-prob-uknown}
Consider any $\bm{\mu} \in \mathcal{M}$ and any $\bm{\pi}$ such that Assumption \ref{ass:action-covering} is satisfied. Let $\bm{\pi}$ be unknown to the learner prior to interacting with the environment. Let $\alpha \in (1, e/2]$. It holds that:
\begin{align}\label{eq:high-prob-unknown}
\mathbb{P}_{\bm{\mu}, \bm{\pi}} \left( \limsup_{\delta \rightarrow 0 } \frac{\tau_\delta}{\log\left( 1/\delta \right)} \le \alpha T^*(\bm{\mu}, \bm{\pi})\right) = 1.
\end{align}
\end{theorem}

Furthermore, as done in the previous section, it is possible to derive a result that directly controls the expectation of the stopping time $\tau_\delta$. 

\begin{theorem}\label{theo:expect-unknown}
Consider any $\bm{\mu} \in \mathcal{M}$ and any $\bm{\pi}$ such that Assumption \ref{ass:action-covering} is satisfied. Let $\bm{\pi}$ be unknown to the learner prior to interacting with the environment. Let $\alpha \in (1, e/2]$. It holds that:
\begin{align}\label{eq:expect-unknown}
\limsup_{\delta \rightarrow 0 } \frac{\mathbb{E}_{\bm{\mu}, \bm{\pi}}[\tau_\delta]}{\log(1/\delta)} \le \alpha T^*(\bm{\mu}, \bm{\pi}).
\end{align}
\end{theorem}

We now proceed by analyzing the results of Theorems \ref{theo:high-prob-uknown} and \ref{theo:expect-unknown}. First of all, as we can appreciate, they fully extend the results of Theorems \ref{theo:high-prob} and \ref{theo:expect} to the unknown policies setting. Notice, in particular, that the theoretical results of the unknown policy setting, i.e., Equations \eqref{eq:high-prob-unknown} and \eqref{eq:expect-unknown}, are \emph{completely equivalent} to the ones previously presented for the case in which $\bm{\pi}$ is available to the learner, i.e., Equations \eqref{eq:high-prob} and \eqref{eq:expect}. Furthermore, as a direct consequence of the fact that Theorem \ref{theo:lower-bound} represents a lower bound to the problem, it follows that the simple modification that we presented at the beginning of this section, is sufficient to derive an \emph{asymptotically optimal} algorithm even in the case in which $\bm{\pi}$ is not available to the learner.  Most importantly, these considerations implies that not knowing $\bm{\pi}$ does not affect the statistical complexity of the problem, at least in the asymptotic regime $\delta \rightarrow 0$. As a direct consequence, all the analysis and discussion we presented in the lower bound section hold equivalently both for the known and the unknown policy settings.

\section*{Appendix E: Proofs and derivations}
\label{app:theorem}

\subsection*{Statistical Complexity}
In this section, we derive claims concerning the statistical complexity of the problem. We begin by proving Theorem \ref{theo:lower-bound}.

\est*
\begin{proof}
Consider an instance $\bm{\lambda} \in \textup{Alt}(\bm{\mu})$. It is easy to see that, from Lemma 1 in \citet{kaufmanncomplexity2016} that:
\begin{align*}
\sum_{e=1}^E \left( \mathbb{E}_{\bm{\mu}, \bm{\pi}} \left[ N_e(\tau_\delta) \right] \sum_{a=1}^K \pi_e(a) d(\mu_a, \lambda_a) \right) \ge \textup{kl}(\delta, 1-\delta),
\end{align*}
where $N_e(t)$ denotes the number of pulls to mediator $e$ at time $t$.
Following \citet{garivieroptimal2016}, we notice that the previous Equation holds for all $\bm{\lambda} \in \textup{Alt}(\bm{\mu})$. Therefore, we have that:
\begin{align*}
\textup{kl}(\delta, 1-\delta) & \le \inf_{\bm{\lambda} \in \textup{Alt}(\bm{\mu})} \mathbb{E}_{\bm{\mu}, \bm{\pi}} [\tau_\delta] \sum_{e=1}^E \left( \frac{\mathbb{E}_{\bm{\mu}, \bm{\pi}} \left[ N_e(\tau_\delta) \right]}{\mathbb{E}_{\bm{\mu}, \bm{\pi}} [\tau_\delta]} \sum_{a=1}^K \pi_e(a) d(\mu_a, \lambda_a) \right) \\ & \le \mathbb{E}_{\bm{\mu}, \bm{\pi}} [\tau_\delta] \sup_{\bm{\omega} \in \Sigma_E} \inf_{\bm{\lambda} \in \textup{Alt}(\bm{\mu})} \sum_{e=1}^E \left( \omega_e \sum_{a=1}^K \pi_e(a) d(\mu_a, \lambda_a)\right),
\end{align*}
thus concluding the proof.
\end{proof}

Given Theorem \ref{theo:lower-bound}, we notice that, from Lemma 3 in \citet{garivieroptimal2016}, we have that $T^*(\bm{\mu}, \bm{\pi})^{-1}$ can be rewritten as:
\begin{align}\label{eq:key}
\min_{a \ne 1} \left( \sum_{e=1}^{E} (\pi_e(1) + \pi_e(a))  \right) I_{\frac{\sum_{e=1}^{E} \pi_e(1)}{\sum_{e=1}^{E} (\pi_e(1) + \pi_e(a))}} (\mu_1, \mu_a),
\end{align} 
where $I_\alpha(\mu_1, \mu_2) \coloneqq \alpha d(\mu_1, \alpha \mu_1 + (1-\alpha) \mu_2) + (1-\alpha) d(\mu_2, \alpha \mu_1 + (1-\alpha) \mu_2)$ denotes a generalized version of the Jensen-Shannon divergence. We notice that, for Gaussian distributions with unitary variance, Equation \eqref{eq:key}, reduces to \citep[see, e.g., Appendix A.4 in ][]{garivieroptimal2016}:
\begin{align}
\min_{a \ne 1} \frac{1}{2} \frac{ \left( \sum_{e=1}^K \pi_e(1) \right) \left( \sum_{e=1}^K \pi_e(a) \right) }{\left( \sum_{e=1}^K \pi_e(1) + \pi_e(a)\right)} \Delta^2_a,
\end{align}
from which the proof of Equation \eqref{eq:justification} and \eqref{eq:lower-bound-single-expert} follows directly.

At this point, we continue by proving Proposition \ref{prop:comparison}.

\comparison*
\begin{proof}
By begin by proving Equation \eqref{eq:prop-one}.
From Equation \eqref{eq:lower-bound-time-ours-2}, we have that, for any mediators' policies $\bm{\pi}$, the following equality holds:
\begin{align*}
T^*(\bm{\mu}, \bm{\pi})^{-1} = \sup_{\bm{\tilde{\pi}} \in \widetilde{\Sigma}_K} \inf_{\bm{\lambda} \in \textrm{Alt}(\bm{\mu})} \sum_{a=1}^K \tilde{\pi}_a d(\mu_a, \lambda_a).
\end{align*}
At this point, we notice that, whenever we consider $\bm{\bar{\pi}}$, $\widetilde{\Sigma}_K$ is equal to $\Sigma_K$.
The result follows by noticing that $\widetilde{\Sigma}_K \subseteq \Sigma_K$ for any mediators' policies $\bm{\pi}$.

We now continue by showing that $T^*(\bm{\mu}, \bm{\pi})^{-1} < T^*(\bm{\mu}, \bm{\bar{\pi}})^{-1}$ holds if and only if $\bm{\omega^*}(\bm{\mu}, \bm{\bar{\pi}}) \notin \widetilde{\Sigma}_K$. First of all, suppose that $T^*(\bm{\mu}, \bm{\pi})^{-1} < T^*(\bm{\mu}, \bm{\bar{\pi}})^{-1}$ holds. However, if $\bm{\omega^*}(\bm{\mu}, \bm{\bar{\pi}}) \in \widetilde{\Sigma}_K$ holds as well, then we would have that $T^*(\bm{\mu}, \bm{\pi})^{-1} \ge T^*(\bm{\mu}, \bm{\bar{\pi}})^{-1}$ by definition of $\bm{\omega^*}(\bm{\mu}, \bm{\bar{\pi}})$. Therefore, $\bm{\omega^*}(\bm{\mu}, \bm{\bar{\pi}}) \notin \widetilde{\Sigma}_K$. On the other hand, if $\bm{\omega^*}(\bm{\mu}, \bm{\bar{\pi}}) \notin \widetilde{\Sigma}_K$ holds, than $T^*(\bm{\mu}, \bm{\pi})^{-1} < T^*(\bm{\mu}, \bm{\bar{\pi}})^{-1}$ follows from the fact that $\bm{\omega^*}(\bm{\mu}, \bm{\bar{\pi}})$ is the unique maximizer of $T^*(\bm{\mu}, \bm{\bar{\pi}})^{-1}$. 
\end{proof}

Finally, before proving Proposition \ref{prop:omega-star-properties}, we first note that the corresponding $\bm{\omega^*}(\bm{\mu}, \bm{{\pi}})$ and $\bm{\tilde{\pi}^*}(\bm{\mu}, \bm{{\pi}})$ are, in general, not unique. As a simple example, it is sufficient to consider the case in which linearly-dependent policies are present in $\bm{\pi}$. At this point, we continue with the proof of Proposition \ref{prop:omega-star-properties}.

\upperhemi*
\begin{proof}
First of all, we prove the convexity of the sets. We begin by recalling the definition of $\bm{\omega}^*(\bm{\mu}, \bm{\pi})$:
\begin{align*}
    \bm{\omega}^*(\bm{\mu}, \bm{\pi}) = \argmax_{\bm{\omega} \in \Sigma_E} \inf_{\bm{\lambda} \in \Lambda(\bm{\mu})} \left(\sum_{e=1}^E \omega_e \sum_{a=1}^K \pi_e(a) d(\mu_a, \lambda_a) \right).
\end{align*}
In other words $\bm{\omega}^*(\bm{\mu}, \bm{\pi})$ is the set of maximizers of an infimum over linear functions (which is well-known to be concave). For this reason $\bm{\omega}^*(\bm{\mu}, \bm{\pi})$ is convex. A similar reasoning can be applied for $\bm{\tilde{\pi}}^*(\bm{\mu}, \bm{\pi})$. \footnote{This argument was used, for instance, in \citet{degenne2019pure}.} At this point, we proceed with the upper-hemicontinuity.
First of all, consider:
\begin{align*}
    f\left(\bm{\mu}, \bm{\pi}, \bm{\omega} \right) = \inf_{\bm{\lambda} \in \textrm{Alt}(\bm{\mu})} \sum_e \omega_e \sum_a \pi_e(a) d(\mu_a, \lambda_a).
\end{align*}
The function $f : \mathcal{M} \times (\Delta_K)^E\times \Delta_E  \rightarrow \mathbb{R}$ is continuous. To see this, from Equation \eqref{eq:key}, we can rewrite $f$ as:
\begin{align*}
        f\left(\bm{\mu}, \bm{\pi}, \bm{\omega} \right) = \min_{a \ne 1} \left( \sum_e \omega_e (\pi_e(1) + \pi_e(a))\right) I_{\frac{\sum_e \omega_e \pi_e(1) }{\sum_e \omega_e (\pi_e(1) + \pi_e(a))}}\left( \mu_1, \mu_a \right).
\end{align*}
Therefore, $f$ can be expressed as a minimum over continous functions. It follows that $f$ is continuous as well. At this point, the proof follows from an application of the Berge's Theorem \citep{aubin2009set} (see e.g., Theorem 22 in \citet{degenne2019pure}). Adopting the same notation as in \cite{degenne2019pure}, consider $\mathbb{X} = \mathcal{M} \times (\Delta_K)^E$, $\mathbb{Y} = \Delta_E$, $\phi(\bm{\mu}, \bm{\pi}) = \Delta_E$, and $u((\bm{\mu}, \bm{\pi}), \bm{\omega}) = f(\bm{\mu}, \bm{\pi}, \bm{\omega})$. At this point, we notice that $\phi$ is compact valued and continuous (since it is constant), while $u$ is continuous. Therefore, due to Berge's Theorem $\left(\bm{\mu}, \bm{\pi} \right) \rightarrow \bm{\omega}^*(\bm{\mu}, \bm{\pi})$ is upper hemicontinous and compact-valued. An identical reasoning can be applied for $\bm{\tilde{\pi}}^*(\bm{\mu}, \bm{\pi})$ replacing $\mathbb{Y}$ and $\phi(\bm{\mu}, \bm{\pi})$ with $\widetilde{\Sigma}_K$. 
\end{proof}

\subsection*{Helper Lemmas}
Before diving into the details of our analysis, we report a known result on C-tracking.
 
\begin{lemma}[Lemma 7 in \cite{garivieroptimal2016}]\label{lemma:expert-c-tracking}
For all $t > 1$ and $e \in [E]$, the C-tracking rules ensures that $N_e(t) \ge \sqrt{t + E^2} - 2E$. Furthermore, consider a sequence $(\bm{\omega}_t)$ such $\bm{\omega}_t \in \bm{\omega}^*(\hat{\bm{\mu}}(t), \bm{\pi})$ for all $t$. Then, C-tracking ensures that:
\begin{align*}
||N_t^E - \sum_{s=0}^{t-1} \bm{\omega}_s ||_{\infty} \le E(1+\sqrt{t}),
\end{align*}
where $N_t^E = (N_1(t), \dots, N_E(t))$ denotes the number of pulls to each mediator at time $t$.
\end{lemma}

Furthermore, we notice that the fact $\mathbb{P}_{\bm{\mu}, \bm{\pi}} \left( \tau_\delta < +\infty, \hat{a}_{\tau_\delta} \ne a^* \right) \le \delta$ holds, is a direct consequence of Assumption 1, forced-exploration, and Proposition 12 in \citet{garivieroptimal2016}.

\subsection*{Algorithm analysis (Known Mediators' Policies; Almost Surely Convergence)}

At this point, we proceed by deriving the almost surely convergence result (i.e, Theorem \ref{theo:high-prob}).

\begin{lemma}\label{lemma:w-convergence-to-convex-set}
Consider a sequence of $\left( \bm{\hat{\bm{\mu}}}(t) \right)_{t \in \mathbb{N}}$ that converges almost surely to $\bm{\mu}$. For all $t \in \mathbb{N}$, let $\bm{\omega}_t \in \bm{\omega}^*(\bm{\hat{\mu}}(t), \bm{\pi})$ be arbitrary oracle weights for $\bm{\hat{\mu}}(t)$ and $\bm{\pi}$. Then, the following holds:
\begin{align*}
    \mathbb{P}_{\bm{\mu}, \bm{\pi}} \left( \lim_{t \rightarrow +\infty} \inf_{\bm{\omega} \in \bm{\omega}^*(\bm{\mu}, \bm{\pi})} \left\Vert \frac{1}{t} \sum_{s=0}^{t-1} \bm{\omega}_s - \bm{\omega} \right\Vert_\infty = 0 \right) = 1
\end{align*}
\end{lemma}
\begin{proof}
The proof follows the one of Lemma 6 of \cite{degenne2019pure}.

Let $\mathcal{E}$ be the following event:
\begin{align*}
    \mathcal{E} = \left\{ \bm{\hat{\mu}}(t)  \rightarrow \bm{\mu} \right\}.
\end{align*}
Event $\mathcal{E}$ holds by assumption with probability $1$. Due to Proposition \ref{prop:omega-star-properties}, we also have that there for all $\epsilon > 0$, there exists $\xi > 0$ such that if $\left\Vert \bm{\hat{\mu}}(t) - \bm{\mu} \right\Vert_\infty \le \xi$ holds, then for all $\bm{\omega}_t \in \bm{\omega}^*(\bm{\hat{\mu}}(t), \bm{\pi})$, $\inf_{\bm{\omega} \in \bm{\omega}^*(\bm{\omega}, \bm{\pi})}\left\Vert \bm{\omega}_t - \bm{\omega} \right\Vert_\infty$ holds as well. At this point, we also notice that, on $\mathcal{E}$, for any $\xi > 0$, there exists $t_0$ such that for all $t \ge t_0$, $\left\Vert \bm{\hat{\mu}}(t) - \bm{\mu} \right\Vert_\infty \le \xi$ holds.

At this point, for any $\bm{\omega} \in \bm{\omega}^*(\bm{\mu}, \bm{\omega})$, we have that, for all $t \ge t_0$:
\begin{align*}
    \left\Vert \frac{1}{t} \sum_{s=0}^{t-1} \bm{\omega}_s - \bm{\omega} \right\Vert_\infty \le \frac{t_0}{t} + \frac{t-t_0}{t} \left\Vert \frac{1}{t - t_0} \sum_{t=t_0}^{t-1} \bm{\omega}_s - \bm{\omega} \right\Vert.
\end{align*}

Taking infimums and using the convexity of $\bm{\omega}^*(\bm{\mu})$ (Proposition \ref{prop:omega-star-properties}), together with Lemma 33 of \cite{degenne2019pure}, we have that:
\begin{align}
    \inf_{\bm{\omega} \in \bm{\omega}^*(\bm{\mu}, \bm{\pi})}\left\Vert \frac{1}{t} \sum_{s=0}^{t-1} \bm{\omega}_s - \bm{\omega} \right\Vert_\infty \le \frac{t_0}{t} + \epsilon,
\end{align}
which concludes the proof.
\end{proof}

A direct consequence of Lemma \ref{lemma:w-convergence-to-convex-set} is the following one.

\begin{lemma}\label{lemma:ne-convergence-to-inf-best-weigh}
Consider a sequence of $\left( \bm{\hat{\bm{\mu}}}(t) \right)_{t \in \mathbb{N}}$ generated while following the C-tracking sampling strategy. Then, it holds that:
\begin{align*}
    \mathbb{P}_{\bm{\mu}, \bm{\pi}} \left( \lim_{t \rightarrow +\infty} \inf_{\bm{\omega} \in \bm{\omega}^*(\bm{\mu}, \bm{\pi})} \left\Vert \frac{N_t^E}{t} - \bm{\omega} \right\Vert_\infty = 0 \right) = 1.
\end{align*}
\end{lemma}

\begin{proof}
    Consider any $\bm{\omega} \in \bm{\omega}^*(\bm{\mu}, \bm{\pi})$. Then, for any sequence $\left( \bm{\omega}_t \right)_{t \in \mathbb{N}}$ such that $\bm{\omega}_t \in \bm{\omega}^*(\bm{\hat{\mu}}, \bm{\pi})$,  we have that:
    \begin{align*}
        \left\Vert \frac{N_t^E}{t} - \bm{\omega} \right\Vert_{\infty} & \le \left\Vert \frac{N_t^E}{t} - \frac{1}{t} \sum_{s=0}^{t-1} \bm{\omega}_s \right\Vert_{\infty} + \left\Vert \frac{1}{t} \sum_{s=0}^{t-1} \bm{\omega}_s - \bm{\omega} \right\Vert_{\infty} \\ & \le \frac{E(1+ \sqrt{t})}{t} + \left\Vert \frac{1}{t} \sum_{s=0}^{t-1} \bm{\omega}_s - \bm{\omega} \right\Vert_{\infty},
    \end{align*}
    where in the second inequality we have used Lemma \ref{lemma:expert-c-tracking}. Then, taking infimums and applying Lemma \ref{lemma:w-convergence-to-convex-set} concludes the proof.
\end{proof}

At this point, we are ready to state the main Lemma that allows us to match, almost surely, the expected lower bound on the sample complexity. 

\begin{lemma}\label{lemma:na-convergence-to-inf-best-weigh}
    Consider a sequence of $\left( \bm{\hat{\bm{\mu}}}(t) \right)_{t \in \mathbb{N}}$ generated while following the C-tracking sampling strategy. Then, it holds that:
\begin{align*}
    \mathbb{P}_{\bm{\mu}, \bm{\pi}} \left( \lim_{t \rightarrow +\infty} \inf_{\bm{\tilde{\pi}} \in \bm{\tilde{\pi}}^*(\bm{\mu}, \bm{\pi})} \left\Vert \frac{N_t^A}{t} - \bm{\tilde{\pi}} \right\Vert_\infty = 0 \right) = 1.
\end{align*}
\end{lemma}
\begin{proof}
We need to show that, with probability $1$, for all $\epsilon > 0$, there exists $t_\epsilon$ such that for all $t \ge t_\epsilon$ the following holds:
\begin{align}\label{eq:na-converngece-to-inf-eq-1}
    \inf_{\bm{\tilde{\pi}} \in \bm{\tilde{\pi}}^*(\bm{\mu}, \bm{\pi})} \left\Vert \frac{N_t^A}{t} - \bm{\tilde{\pi}} \right\Vert_\infty \le \epsilon.
\end{align}
We notice that a sufficient condition for Equation \eqref{eq:na-converngece-to-inf-eq-1} to hold is that it holds for some policy $\bm{\tilde{\pi}} \in \bm{\tilde{\pi}}^*(\bm{\mu}, \bm{\pi})$ (i.e., not necessarly the one that attains the infimum). 

At this point, we proceed with some considerations. First of all, we know that, due to the law of large numbers, and the fact that $N_e(t) \rightarrow +\infty$, we have that $\frac{N_(e,a)(t)}{N_e(t)} \rightarrow \pi_e(a)$ with probability $1$. More precisely, with probability $1$, for all $\epsilon_1 > 0$, there exists $t_{\epsilon_1}$ such that for all $t \ge t_{\epsilon_1}$ the following holds:
\begin{align*}
    \frac{N_{e,a}(t)}{N_e(t)} \in \left[ \pi_e(a) - \epsilon_1, \pi_e(a) + \epsilon_1 \right].
\end{align*}
Furthermore, due to Lemma \ref{lemma:ne-convergence-to-inf-best-weigh}, we have that, with probability $1$, for all $\epsilon_2 > 0$, there exists $t_{\epsilon_2}$, such that for all $t \ge t_{\epsilon_2}$ the following holds:
\begin{align*}
    \inf_{\bm{\omega} \in \bm{\omega}^*(\bm{\mu}, \bm{\pi})} \left\Vert \frac{N_e(t)}{t} - \omega_e \right\Vert_\infty \in \left[-\epsilon_2, \epsilon_2 \right].
\end{align*}
In other words, let $\bm{\omega}_t \in \argmin_{\bm{\omega} \in \bm{\omega}^*(\bm{\mu},\bm{\pi})} \left\Vert \frac{N_e(t)}{t} - \omega_e \right\Vert_\infty$. The following holds:
\begin{align*}
    \frac{N_e(t)}{t} \in \left[ \omega_{t,e} - \epsilon_2, \omega_{t,e} + \epsilon_2 \right]. 
\end{align*}

At this point, focus on Equation \eqref{eq:na-converngece-to-inf-eq-1}. Let $\bm{\tilde{\pi}}$ be any policy within $\bm{\tilde{\pi}}^*(\bm{\mu}, \bm{\pi})$. Then, Equation \eqref{eq:na-converngece-to-inf-eq-1} is satisfied whenever for all actions $a \in \mathcal{A}$ the following holds:
\begin{align*}
    \frac{N_a(t)}{t} - \tilde{\pi}(a) \le \epsilon,
\end{align*}
and
\begin{align*}
    -\frac{N_a(t)}{t} + \tilde{\pi}(a) \le \epsilon,
\end{align*}
holds. Let us focus on $\frac{N_a(t)}{t} - \tilde{\pi}(a) \le \epsilon$ (the case of $-\frac{N_a(t)}{t} + \tilde{\pi}(a) \le \epsilon$ is almost identical):
\begin{align*}
    \frac{N_a(t)}{t} - \tilde{\pi}(a) = \sum_e \frac{N_{a,e}(t)}{N_e(t)} \frac{N_e(t)}{t} - \sum_e \tilde{\omega}_e \pi_e(a),
\end{align*}
where $\bm{\tilde{\omega}}$ is any oracle weight that induces $\bm{\tilde{\pi}}$. With probability $1$, however, we have that:
\begin{align*}
    \frac{N_a(t)}{t} - \tilde{\pi}(a) & = \sum_e \frac{N_{a,e}(t)}{N_e(t)} \frac{N_e(t)}{t} - \sum_e \tilde{\omega}_e \pi_e(a) \\ & \le \sum_e (\pi_e(a) + \epsilon_1) (\omega_{t,e} + \epsilon_2) - \sum_e \tilde{\omega}_e \pi_e(a).
\end{align*}
At this point, we notice that the previous equation holds for any oracle weight $\bm{\tilde{\omega}}$ that induces $\bm{\tilde{\pi}}$, and furthermore, it also holds for any $\bm{\tilde{\pi}} \in \bm{\tilde{\pi}}^*(\bm{\mu}, \bm{\pi})$. It thus sufficies to pick $\bm{\tilde{\omega}}$ equal to $\bm{\omega}_t$ to obtain the following:
\begin{align*}
    \frac{N_a(t)}{t} - \tilde{\pi}_t(a) \le \sum_e \pi_e(a) \epsilon_2 + \epsilon_1 \omega_{t,e} + \epsilon_1 \epsilon_2, 
\end{align*}
which concludes the proof.

\end{proof}

We are now ready to prove Theorem \ref{theo:high-prob}.

\highprob*
\begin{proof}
    Consider the following event: 
    \begin{align*}
        \mathcal{E} = \left\{ \lim_{t \rightarrow +\infty} \inf_{\bm{\tilde{\pi}} \in \bm{\tilde{\pi}}^*(\bm{\mu}, \bm{\pi})} \left\Vert \frac{N_t^A}{t} - \bm{\tilde{\pi}} \right\Vert_\infty = 0 \textrm{ and } \hat{\bm{\mu}}(t) \rightarrow \bm{\mu} \right\}
    \end{align*}
    Due to Lemma \ref{lemma:na-convergence-to-inf-best-weigh}, the sampling strategy, the assumption on the mediators's policies, and the law of large numbers we know that $\mathcal{E}$ is of probability $1$. Therefore, there exists $t_0$ such that for all $t \ge t_0$, $\hat{\mu}_1(t) > \max_{a \ne 1} \hat{\mu}_a(t)$ and, consequently: 
    \begin{align*}
        Z(t) = t \left[ \min_{a \ne 1} \left( \frac{N_1(t)}{t} + \frac{N_a(t)}{t} \right) I_{\frac{N_1(t)/t}{N_1(t)/t + N_a(t)/t}} (\hat{\mu}_1(t)), \hat{\mu}_a(t)) \right].
    \end{align*}
    

    Let $\bm{\tilde{\pi}}_t \in \arginf_{\bm{\tilde{\pi}} \in \bm{\tilde{\pi}}^*(\bm{\mu}, \bm{\pi})} \left\Vert \frac{N_t^A}{t} - \bm{\tilde{\pi}} \right\Vert_\infty $.
    For all $\epsilon > 0$, there exists $t_1 \ge t_0$ such that for all $t \ge t_1$ and all action $a \in \mathcal{A} \setminus \{ 1 \}$ the following holds:
    \begin{align*}
         \left( \frac{N_1(t)}{t} + \frac{N_a(t)}{t} \right) I_{\frac{N_1(t)/t}{N_1(t)/t + N_a(t)/t}} (\hat{\mu}_1(t)), \hat{\mu}_a(t)) \ge \frac{\tilde{\pi}_t(1) + \tilde{\pi}_t(a)}{1+\epsilon} I_{\frac{\tilde{\pi}_t(1)}{\tilde{\pi}_t(1) + \tilde{\pi}_t(a)}}(\mu_1, \mu_a).
    \end{align*}
    Therefore, since $\bm{\tilde{\pi}}_t \in \bm{\tilde{\pi}}^*(\bm{\mu}, \bm{\pi})$, for all $t \ge t_1$ we have that:
    \begin{align*}
        Z(t) \ge \frac{t}{(1+\epsilon) T^*(\bm{\mu}, \bm{\pi})}.
    \end{align*}
    The rest of the proof follows unchanged w.r.t. Proposition 13 in \cite{garivieroptimal2016}.
    
\end{proof}

\subsection*{Algorithm analysis (Known Mediators' Policies; Expectation)}

In order to prove Theorem \ref{theo:expect}, we begin with some concentration events analysis.

First of all, let $h(T) = T^{1/4}$, and $\epsilon > 0$. Define:
\begin{align*}
    \mathcal{E}_T = \bigcap_{t=h(T)}^T \left( \left\Vert \bm{\hat{\mu}}(t) - \bm{\mu} \right\Vert_{\infty} \le \xi \right),
\end{align*}
where $\xi$ is such that:
\begin{align*}
\left\Vert \bm{\mu}' - \bm{\mu} \right\Vert_{\infty} \le \xi \implies \forall \bm{\omega}' \in \bm{\omega}^*(\bm{\mu}', \bm{\pi}) \textrm{ } \exists \bm{\omega} \in \bm{\omega}^*(\bm{\mu},\bm{\pi}), \left\Vert \bm{\omega}' - \bm{\omega} \right\Vert_\infty \le \epsilon.
\end{align*}

\begin{lemma}\label{lemma:mu-concentration}
There exists two constants $B$ and $C$ such that:
\begin{align*}
    \mathbb{P}_{\bm{\mu}, \bm{\pi}}(\mathcal{E}_T^c) \le BT \exp(-CT^{1/8}).
\end{align*}
\end{lemma}
\begin{proof}
Let $T$ be such that $h(T) \ge E^2$. Let $p_{\textrm{min}}(a) = \min_e \pi_e(a)$, and define the the event $\mathcal{J}_T$ as 
\begin{align*}
\mathcal{J}_T = \bigcap_{t=h(T)}^T \bigcap_{a=1}^K \left\{ N_a(t) \ge \frac{1}{4} p_{\textrm{min}}(a) \min_e N_e(t) \right\}.  
\end{align*}
At this point, from the tower rule, we have that:
\begin{align*}
    \mathbb{P}_{\bm{\mu}, \bm{\pi}}(\mathcal{E}_T^c) & = \mathbb{E}_{\bm{\mu}, \bm{\pi}} \left[ \bm{1}\left\{ \mathcal{E}_T^c \right\} \right]  \\ & = \mathbb{E}_{\bm{\mu}, \bm{\pi}} \left[ \bm{1}\left\{ \mathcal{E}_T^c \right\} | \mathcal{J}_T \right] \mathbb{P}_{\bm{\mu}, \bm{\pi}}(\mathcal{J}_T) + \mathbb{E}_{\bm{\mu}, \bm{\pi}} \left[ \bm{1}\left\{ \mathcal{E}_T^c \right\} | \mathcal{J}_T^c \right] \mathbb{P}_{\bm{\mu}, \bm{\pi}}(\mathcal{J}_T^c) \\ & \le \mathbb{E}_{\bm{\mu}, \bm{\pi}} \left[ \bm{1}\left\{ \mathcal{E}_T^c \right\} | \mathcal{J}_T \right] + \mathbb{P}_{\bm{\mu}, \bm{\pi}}(\mathcal{J}_T^c).
\end{align*}

At this point, we first focus on $\mathbb{E}_{\bm{\mu}, \bm{\pi}} \left[ \bm{1}\left\{ \mathcal{E}_T^c \right\} | \mathcal{J}_T \right]$. Due to the forced tracking (Lemma \ref{lemma:expert-c-tracking}), we know that, under $\mathcal{J}_T$, the following holds:
\begin{align*}
    N_a(t) \ge \frac{1}{4} p_{\textrm{min}}(a) \min_e N_e(t) \ge \frac{1}{4} p_{\textrm{min}}(a) \left( \sqrt{t + E^2} - 2E \right). 
\end{align*}
Therefore, from Lemma 19 of \cite{garivieroptimal2016}, we have that:
\begin{align*}
    \mathbb{E}_{\bm{\mu}, \bm{\pi}} \left[ \bm{1}\left\{ \mathcal{E}_T^c \right\} | \mathcal{J}_T \right] \le B_1 T \exp \left( -C_1 T^{1/8} \right).
\end{align*}

We now continue with bounding $\mathbb{P}_{\bm{\mu}, \bm{\pi}}(\mathcal{J}_T^c)$. From Boole's inequality we have that:
\begin{align*}
    \mathbb{P}_{\bm{\mu}, \bm{\pi}}(\mathcal{J}_T^c) \le \sum_{t=h(T)}^T \sum_{a=1}^K \mathbb{P}_{\bm{\mu}, \bm{\pi}}\left(N_a(t) \le \frac{1}{4} p_{\textrm{min}}(a) \min_e N_e(t)) \right).
\end{align*}

For each time $t$, let $E_t$ be the mediator selected at time $t$ by the algorithm. Then,
\begin{align*}
    \mathbb{P}_{\bm{\mu}, \bm{\pi}}(\mathcal{J}_T^c) \le \sum_{t=h(T)}^T \sum_{a=1}^K \mathbb{P}_{\bm{\mu}, \bm{\pi}}\left(N_a(t) -\frac{1}{2} \sum_{s=0}^{t-1} \pi_{E_s}(a) \le \frac{1}{4} p_{\textrm{min}}(a) \min_e N_e(t) -\frac{1}{2} \sum_{s=0}^{t-1} \pi_{E_s}(a)) \right).
\end{align*}
Notice that, by definition we have that:
\begin{align*}
    \frac{1}{2} \sum_{s=0}^{t-1} \pi_{E_s}(a) = \frac{1}{2} \sum_{s=0}^{t-1} \sum_{e=1}^E \bm{1}\left\{E_s = e \right\} \pi_e(a) = \frac{1}{2} \sum_{e=1}^E N_e(t) \pi_e(a) \ge \frac{1}{2} p_{\textrm{min}}(a) \min_e N_e(t).
\end{align*}
Therefore, we have that:
\begin{align*}
    \mathbb{P}_{\bm{\mu}, \bm{\pi}}(\mathcal{J}_T^c) & \le \sum_{t=h(T)}^T \sum_{a=1}^K \mathbb{P}_{\bm{\mu}, \bm{\pi}}\left(N_a(t) -\frac{1}{2} \sum_{s=0}^{t-1} \pi_{E_s}(a) \le -\frac{1}{4} p_{\textrm{min}}(a) \min_e N_e(t) \right) \\ & \sum_{t=h(T)}^T \sum_{a=1}^K \mathbb{P}_{\bm{\mu}, \bm{\pi}}\left(N_a(t) -\frac{1}{2} \sum_{s=0}^{t-1} \pi_{E_s}(a) \le -\frac{1}{4} p_{\textrm{min}}(a) \left( \sqrt{t} - E \right) \right),
\end{align*}
where in the last inequality we have used Lemma \ref{lemma:expert-c-tracking} together with $h(T) \ge E^2$. At this point, applying Lemma F.4 of \citet{dann2017unifying}, we obtain:
\begin{align*}
    \mathbb{P}_{\bm{\mu}, \bm{\pi}}(\mathcal{J}_T^c) & \le \sum_{t=h(T)}^T \sum_{a=1}^K \exp \left( -\frac{1}{2} p_{\textrm{min}}(a) \left( \sqrt{t} - E \right) \right) \\ & \le \sum_{t=h(T)}^T B_2 \exp \left( -C_2 \sqrt{t}\right) \\ & \le B_2 T \exp\left( -C_2 T^{1/8} \right), 
\end{align*}
which concludes the proof.
\end{proof}

We now continue by defining another event that is crucial to our analysis. For all $t \ge 0$, we denote with $E_t$ the mediator that is played at time $t$ by the algorithm. Then, for some $\gamma \in (\frac{1}{2}, 1)$ and some constant $c$, we define:
\begin{align*}
    \mathcal{E}_T'(\gamma) = \bigcap_{t=h(T)}^T \bigcap_{a=1}^K \left\{ |N_a(t) - \sum_{s=0}^{t-1} \pi_{E_s}(a)| \le c t^{\gamma} \right\}.
\end{align*}

\begin{lemma}\label{lemma:na-concetration}
Let $\gamma=\frac{3}{4}$. There exists two constants $B$ and $C$ such that:
\begin{align*}
    \mathbb{P}_{\bm{\mu}, \bm{\pi}}(\mathcal{E}_T'^{c}(\gamma)) \le BT \exp(-CT^{1/8}).
\end{align*}
\end{lemma}
\begin{proof}
We have that:
\begin{align*}
    \mathbb{P}_{\bm{\mu}, \bm{\pi}}(\mathcal{E}_T'^{c}) & \le \sum_{t=h(T)}^T \sum_{a=1}^K \mathbb{P}_{\bm{\mu}, \bm{\pi}}\left( |N_a(t) - \sum_{s=0}^{t-1} \pi_{E_s}(a)| \le c t^{\gamma} \right) \\ & \le \sum_{t=h(T)}^T \sum_{a=1}^K \exp\left( \frac{-c^2 t^{2\gamma-1}}{2} \right) \\ & = \sum_{t=h(T)}^T K \exp \left( \frac{-c^2 \sqrt{t}}{2}\right) \\ & = \sum_{t=h(T)}^T B \exp \left( -C \sqrt{t} \right) \\ & \le BT \exp \left( -C T^{1/8} \right),
\end{align*}
where, in the first step we have used Boole's inequality, in the second one Azuma-Hoeefding, in the third one have used $\gamma = \frac{3}{4}$, and in the fourth one we have redefined the constants. 
\end{proof}

At this point, given our events, we directly inherit Lemma from \citet{degenne2019pure}.

\begin{lemma}[Lemma 35 in \citet{degenne2019pure}]\label{lemma:finite-time-w-convergence}
    There exists a constant $T_\epsilon$ such that for $T \ge T_\epsilon$ it holds that on $\mathcal{E}_T$, C-tracking verifies:
    \begin{align*}
        \forall t \ge \sqrt{T}, \quad \inf_{\bm{\omega} \in \bm{\omega}^*(\bm{\mu}, \bm{\pi})} \left\Vert \frac{N_t^E}{t} - \bm{\omega} \right\Vert_\infty \le 3\epsilon.
    \end{align*}
\end{lemma}

We are now ready to state the equivalent of Lemma \ref{lemma:na-convergence-to-inf-best-weigh} for the analysis of $\mathbb{E}_{\bm{\mu}, \bm{\pi}}[\tau_\delta]$.

\begin{lemma}\label{lemma:finite-time-pi-convergence}
There exists a constant $T_\epsilon$ such that for $T \ge T_\epsilon$ it holds that on $\mathcal{E}_T \cap \mathcal{E}_T'(3/4)$, C-tracking verifies:
    \begin{align*}
        \forall t \ge \sqrt{T}, \quad \inf_{{\bm{\tilde{\pi}} \in \bm{\tilde{\pi}}^*(\bm{\mu}, \bm{\pi})}} \left\Vert \frac{N_t^A}{t} - \bm{\tilde{\pi}} \right\Vert_\infty \le 2E\epsilon.
    \end{align*}
\end{lemma}
\begin{proof}
    Consider $T$ such that the condition that defines Lemma \ref{lemma:finite-time-w-convergence} is satisfied. At this point, focus on $t \ge h(T)^2$. For any action $a \in \mathcal{A}$ and  $\bm{\tilde{\pi}} \in \bm{\tilde{\pi}}^*(\bm{\mu}, \bm{\pi})$ we have that:
    \begin{align*}
        \Big| \frac{N_a(t)}{t} - \tilde{\pi}(a) \Big| & \le \Big| \frac{N_a(t)}{t} - \frac{1}{t} \sum_{s=0}^{t-1} \pi_{E_s}(a) \Big| + \Big| \frac{1}{t} \sum_{s=0}^{t-1} \pi_{E_s}(a) - \tilde{\pi}(a) \Big| \\ & \le \frac{ct^{3/4}}{t} + \frac{h(T)}{t} + \Big| \frac{1}{t} \sum_{s=h(T)}^{t-1} (\pi_{E_s}(a) - \tilde{\pi}(a)) \Big| \\ & \le \frac{1}{T^{1/4}} + \frac{c}{T^{1/4}} + \Big| \frac{1}{t} \sum_{s=h(T)}^{t-1} (\pi_{E_s}(a) - \tilde{\pi}(a)) \Big|,
    \end{align*}
    where the first inequality follows from triangular decomposition, the second one, instead, by definition of $\mathcal{E}_T'$, and the third one by $t \ge h(T)^2$. 
    At this point, however, we notice that:
    \begin{align*}
        \Big| \frac{1}{t} \sum_{s=0}^{t-1} (\pi_{E_s}(a) - \tilde{\pi}(a)) \Big| = \Big| \sum_{e=1}^E \pi_e(a) \frac{N_e(t)}{t} - \tilde{\pi}(a) \Big|. 
    \end{align*}
    Therefore, we have that: 
    \begin{align*}
        \left\Vert \frac{N_t^A}{t} - \bm{\tilde{\pi}} \right\Vert_\infty & \le \frac{c+1}{T^{1/4}} + \max_a \Big| \sum_e \pi_e(a) \frac{N_e(t)}{t} - \sum_e \pi_e(a) \tilde{\omega}_e \Big| \\ & \le \frac{c+1}{T^{1/4}} + E \left\Vert \frac{N_t^E}{t} - \bm{\tilde{\omega}} \right\Vert_\infty,
    \end{align*}
    where $\bm{\tilde{\omega}}$ is any weights vector that induces $\bm{\tilde{\pi}}$. Taking infimums, we obtain:
    \begin{align*}
        \inf_{{\bm{\tilde{\pi}} \in \bm{\tilde{\pi}}^*(\bm{\mu}, \bm{\pi})}} \left\Vert \frac{N_t^A}{t} - \bm{\tilde{\pi}} \right\Vert_\infty \le \frac{c+1}{T^{1/4}} + E \inf_{\bm{\omega} \in \bm{\omega}^*(\bm{\mu}, \bm{\pi})} \left\Vert \frac{N_t^E}{t} - \bm{\omega} \right\Vert_\infty.
    \end{align*}
    Applying Lemma \ref{lemma:finite-time-w-convergence} concludes the proof.
\end{proof}

At this point, we are ready to analyze the sample complexity of Track and Stop within our peculiar setting. First of all, we begin a known result widely adopted within the TaS literature (e.g., Lemma $13$ of \cite{degenne2019pure}).
\begin{lemma}\label{lemma:tech-lemma-1}
    Suppose there exists $T_0 \in \mathbb{N}$ such that, for all $T \ge T_0$, $\mathcal{E}_T \cap \mathcal{E}_T' \subset{\tau_\delta \le T}$. Then,
    \begin{align*}
        \mathbb{E}_{\bm{\mu}, \bm{\pi}}\left[ \tau_\delta \right] \le T_0 + \sum_{t=T_0}^{+\infty} \mathbb{P}_{\bm{\mu}, \bm{\pi}}(\mathcal{E}_T^c) + \sum_{t=T_0}^{+\infty} \mathbb{P}_{\bm{\mu}, \bm{\pi}}(\mathcal{E}_T'^{c}).
    \end{align*}
\end{lemma}

Before proceeding within the analysis, we notice that the sums that depends on the events $\mathcal{E}_T^c$ and $\mathcal{E}_T'^c$ are finite. This is a direct consequence of Lemmas \ref{lemma:mu-concentration}
and \ref{lemma:na-concetration}.

We proceed by providing a suitable $T_0$ that can be used within Lemma \ref{lemma:tech-lemma-1}.
We begin with some definition. For any $\bm{\mu}$ and $\bm{\tilde{\pi}}$, we define:
\begin{align*}
    g(\bm{\mu}, \bm{\tilde{\pi}}) = \min_{a \ne 1} (\tilde{\pi}_1 + \tilde{\pi}_a) I_{\frac{\tilde{\pi}_1}{\tilde{\pi}_1 + \tilde{\pi}_a}}(\mu_1, \mu_a).
\end{align*}
Then, for any $\epsilon, \xi > 0$, we introduce the quantity:
\begin{align*}
    C^*_{\epsilon, \xi}(\bm{\mu}) = \inf_{\substack{_{\bm{\mu}' : \left\Vert \bm{\mu} - \bm{\mu}' \right\Vert_\infty \le \xi} \\ \bm{\tilde{\pi}}' : \inf_{\bm{\tilde{\pi}} \in \bm{\tilde{\pi}}^*(\bm{\mu}, \bm{\pi}) \left\Vert \bm{\tilde{\pi}}' - \bm{\tilde{\pi}}  \right\Vert_\infty \le 3E\epsilon }}} g(\bm{\mu}', \bm{\tilde{\pi}}').
\end{align*}

\begin{lemma}\label{lemma:tech-lemma-2}
    Suppose there exists $T_1 \in \mathbb{N}$ such that, for all $T \ge T_1$, if $\mathcal{E}_T \cap \mathcal{E}_T'$ holds, then $g(\bm{\hat{\mu}}(t), \frac{N_t^A}{t}) \ge t C^*_{\epsilon, \xi}(\bm{\mu})$ holds as well for $t \ge \sqrt{T}$. Then, under that event:
    \begin{align*}
        \tau_\delta \le T_0 = \max \left\{ T_1, \inf \left\{ T \in \mathbb{N}: \sqrt{T} + \frac{\beta(T, \delta)}{C^*_{\epsilon, \xi}(\bm{\mu})} \le T\right\} \right\}.
    \end{align*}
\end{lemma}

\begin{proof}
    Set $T \ge T_1$ and suppose that $\mathcal{E}_T \cap \mathcal{E}_T'$ holds. 
    \begin{align*}
        \min \left\{ \tau_\delta, T \right\} & \le \sqrt{T} + \sum_{t=\sqrt{T}}^T \bm{1} \left\{ \tau_\delta > t \right\} \\ & \le \sqrt{T} + \sum_{t=\sqrt{T}}^T \bm{1} \left\{ Z(t) \le \beta(t, \delta) \right\} \\ & \le \sqrt{T} + \sum_{t=\sqrt{T}}^T \bm{1}\left\{ t \le \frac{\beta(t,\delta)}{C^*_{\epsilon, \xi}(\bm{\mu})} \right\} \\ & \le \sqrt{T} + \frac{\beta(T, \delta)}{C^*_{\epsilon, \xi}(\bm{\mu})}, 
    \end{align*}
    where in the third inequality we have used the fact that, under $\mathcal{E}_T \cap \mathcal{E}_T'$, $g(\bm{\hat{\mu}}(t), \frac{N_t^A}{t}) \ge t C^*_{\epsilon, \xi}(\bm{\mu})$ holds for $t \ge \sqrt{T}$. The other steps follows from simple algebraic manipulations.

    The statement follows directly from the definition of $T_0$.
\end{proof}

At this point, we can prove Theorem \ref{theo:expect}.

\expect*
\begin{proof}
    Using Lemma \ref{lemma:finite-time-pi-convergence}, for $T \ge T_\epsilon$, on the event $\mathcal{E}_T \cap \mathcal{E}_T'$ it holds that for every $t \ge \sqrt{T}$:
    \begin{align*}
        Z(t) \ge t C^*_{\epsilon, \xi}(\bm{\mu}).
    \end{align*}
    Therefore, the condition of Lemma \ref{lemma:tech-lemma-2} are satisfied, and we can use $T_0$ defined as in Lemma \ref{lemma:tech-lemma-2} to apply Lemma \ref{lemma:tech-lemma-1}. Therefore, we obtain: 
    \begin{align*}
        \mathbb{E}_{\bm{\mu}, \bm{\pi}}[\tau_\delta] \le T_0 + \sum_{T=1}^{+\infty} \mathbb{P}_{\bm{\mu}, \bm{\pi}}(\mathcal{E}_T^c) + \sum_{T=1}^{+\infty} \mathbb{P}_{\bm{\mu}, \bm{\pi}}(\mathcal{E}_T'^c).
    \end{align*}
    At this point, the rest of the proof follows as the original proof of TaS (Theorem 14 in \citet{garivieroptimal2016}).
\end{proof}

\subsection*{Algorithm analysis (Unknown Mediators' Policies)}
For the case in which the mediators' policies are unknown to the learner, few mofidications are needed for the proof. 

Concerning the analysis of Theorem \ref{theo:high-prob-uknown}, it is sufficient to notice that Lemma \ref{lemma:w-convergence-to-convex-set} holds unchanged for sequences of arbitrary oracle weights $\bm{\omega}_t$ that belongs to $\bm{\omega}^*(\bm{\hat{\mu}}(t), \bm{\hat{\pi}}(t)$. Indeed, consider a sequence $\left( \bm{\hat{\pi}}(t) \right)_{t \in \mathbb{N}}$ that converges to $\bm{\pi}$ almost surely (which is guaranteed by C-tracking and the law of large numbers). Then, it is sufficient to notice that the mapping $(\bm{\omega}, \bm{\pi}) \rightarrow \bm{\omega}^*(\bm{\mu}, \bm{\omega})$ is upper hemicontinuous in both arguments. The rest of the proofs follow directly. 

The problem is slightly more subtle for the analysis of Theorem \ref{theo:expect-unknown}. More specifically, $\mathcal{E}_T$ needs to be redefined in the following way.

\begin{align*}
\mathcal{E}_T = \bigcap_{t=h(T)}^T \left( \left\{ \left\Vert \bm{\hat{\mu}}(t) - \bm{\mu} \right\Vert_\infty \le \xi \right\} \cap \left\{ \left\Vert \bm{\hat{\pi}}(t) - \bm{\pi} \right\Vert_\infty \le \xi \right\} \right),
\end{align*}
where $\xi$ is such that, $\left\Vert \bm{\mu}' - \bm{\mu} \right\Vert_{\infty} \le \xi$ and $\left\Vert \bm{\pi'} - \bm{\pi} \right\Vert_\infty \le \xi$ implies that
\begin{align*}
\forall \bm{\omega}' \in \bm{\omega}^*(\bm{\mu}', \bm{\pi}') \textrm{ } \exists \bm{\omega} \in \bm{\omega}^*(\bm{\mu},\bm{\pi}), \left\Vert \bm{\omega}' - \bm{\omega} \right\Vert_\infty \le \epsilon.
\end{align*}

At this point, we need to control the probability of $\mathcal{E}_T$. 
The following Lemma, combined with the proof of Lemma \ref{lemma:mu-concentration}, directly implies that $\mathbb{P}_{\bm{\mu}, \bm{\pi}}(\mathcal{E}_T^c) \le BT \exp(-CT^{1/8})$ holds.

\begin{lemma}
Consider the event:
\begin{align*}
\mathcal{K}_T = \bigcap_{t=h(T)}^T \left\{ \left\Vert \bm{\hat{\pi}}(t) - \bm{\pi} \right\Vert_\infty \le \xi \right\},
\end{align*}
then, $\mathbb{P}_{\bm{\mu}, \bm{\pi}}(\mathcal{K}_T^c) \le BT \exp(-CT^{1/8})$ holds for some constants $B$ and $C$.
\end{lemma}
\begin{proof}
From Boole's inequality we obtain:
\begin{align*}
\mathbb{P}_{\bm{\mu}, \bm{\pi}}(\mathcal{K}_T^c) \le \sum_{t=h(T)}^T \sum_{e=1}^E \sum_{a=1}^K \mathbb{P}_{\bm{\mu}, \bm{\pi}}\left( \frac{N_{e,a}(t)}{N_e(t)} \le \pi_e(a) - \xi \right) + \mathbb{P}_{\bm{\mu}, \bm{\pi}}\left( \frac{N_{e,a}(t)}{N_e(t)} \ge \pi_e(a) + \xi \right),
\end{align*}
where $N_{e,a}(t)$ denotes the number of pulls to action $a$ when querying mediator $e$ at time $t$. At this point, let $T$ be such that $h(T) \ge E^2$, then due to Lemma \ref{lemma:expert-c-tracking}, $N_e(t) \ge \sqrt{t} - E$ for every mediators' $e$. The result than follows from Lemma 19 in \citet{garivieroptimal2016}.
\end{proof}

The new definition of $\mathcal{E}_T$ plays a role in the equivalent of Lemma \ref{lemma:finite-time-w-convergence} that needs to be derived for the unknown policy setting (for which we report the proof below for completeness). The rest of the proof of Theorem \ref{theo:expect-unknown} follows directly from the ones of Theorem \ref{theo:expect}.

\begin{lemma}
    There exists a constant $T_\epsilon$ such that for $T \ge T_\epsilon$ it holds that on $\mathcal{E}_T$, C-tracking with unknown mediators' policies verifies:
    \begin{align*}
        \forall t \ge \sqrt{T}, \quad \inf_{\bm{\omega} \in \bm{\omega}^*(\bm{\mu})} \left\Vert \frac{N_t^E}{t} - \bm{\omega} \right\Vert_\infty \le 3\epsilon.
    \end{align*}
\end{lemma}
\begin{proof}
Consider any $\bm{\omega} \in \bm{\omega}^*(\bm{\pi}, \bm{\omega})$, and consider a sequence $\bm{\omega}_s \in \bm{\omega}^*(\bm{\hat{\mu}}(s), \bm{\hat{\pi}}(s))$. At this point, we can write:
\begin{align*}
\left\Vert \frac{N_t^E}{t} - \bm{\omega} \right\Vert_\infty & \le \left\Vert \frac{N_t^E}{t} - \frac{1}{t} \sum_{s=0}^{t-1} \bm{\omega}_s \right\Vert_\infty + \left\Vert \frac{1}{t} \sum_{s=0}^{t-1} \bm{\omega}_s - \bm{\omega} \right\Vert_\infty \\ & \le \frac{E(1+\sqrt{t})}{t} + \left\Vert \frac{1}{t} \sum_{s=0}^{t-1} \bm{\omega}_s - \bm{\omega} \right\Vert_\infty \\ & \le \frac{E(1+\sqrt{t})}{t} + \frac{h(T)}{t} + \left\Vert \frac{1}{t} \sum_{s=h(T)}^{t-1} (\bm{\omega}_s - \bm{\omega}) \right\Vert_\infty,
\end{align*}
where we combined simple algebraic manipulations with Lemma \ref{lemma:expert-c-tracking}. However, under event $\mathcal{E}_T$, we have that $\xi$ is such that, $\left\Vert \bm{\mu}' - \bm{\mu} \right\Vert_{\infty} \le \xi$ and $\left\Vert \bm{\pi'} - \bm{\pi} \right\Vert_\infty \le \xi$ implies that:
\begin{align*}
\forall \bm{\omega}' \in \bm{\omega}^*(\bm{\mu}', \bm{\pi}') \textrm{ } \exists \bm{\omega} \in \bm{\omega}^*(\bm{\mu},\bm{\pi}), \left\Vert \bm{\omega}' - \bm{\omega} \right\Vert_\infty \le \epsilon.
\end{align*}
The convexity of $\bm{\omega}^*(\bm{\mu}, \bm{\pi})$ then ensures that $\inf_{\bm{\omega} \in \bm{\omega}^*(\bm{\omega}, \bm{\pi})} \left\Vert \frac{1}{t} \sum_{s=h(T)}^{t-1} (\bm{\omega}_s - \bm{\omega}) \right\Vert_\infty \le \epsilon$ holds as well (Lemma 33 in \citet{garivieroptimal2016}). Therefore, under $\mathcal{E}_T$ we obtain:
\begin{align*}
\inf_{\bm{\omega} \in \bm{\omega}^*{\bm{\mu}, \bm{\pi}}} \left\Vert \frac{N_t^E}{t} - \bm{\omega} \right\Vert_\infty \le \frac{E(1+\sqrt{t})}{t} + \frac{h(T)}{t} + \epsilon,
\end{align*}
which concludes the proof.

\end{proof}

\newpage
\section*{Appendix E: Experiments}

\begin{table}[t]
\centering
\begin{tabular}{l|l|l|l|l|}
	& \textbf{TaS} & \textbf{TaS-MF-k} & \textbf{TaS-MF-u} & \textbf{Uniform Sampling}  \\
    $\delta=0.4$ & $389.12 \pm 3.79$ & $863.87 \pm 13.14$ & $883.52 \pm 18.44$ & $1689.71 \pm 25.98$  \\
	$\delta=0.1$ & $450.12 \pm 3.45$ & $990.27 \pm 14.11$ & $1013.14 \pm 15.89$ & $1891.62 \pm 29.76$  \\
	$\delta=0.01$ & $553.02 \pm 3.45$ & $1179.11 \pm 13.59$ & $1205.61 \pm 16.95$ & $2245.46 \pm 33.03$  \\
	$\delta=0.001$ & $655.03 \pm 3.97$ & $1376.72 \pm 15.94$ & $1378.14 \pm 18.91$ & $2619.26 \pm 37.53$  \\
\end{tabular}
\caption{Experiment results for $100$ runs. The table report mean and $95\%$ confidence intervals of the empirical stopping time.}
\label{table}
\end{table}

To conclude, we report simple numerical experiments.

\subsection*{Experiment 1}
First of all, we analyze empirically the effect of the partial controllability that arises from the mediators' feedback. More specifically, we consider a Gaussian bandit model with $K=4$. The different arms have mean $\bm{\mu} = (1.5, 1.0, 0.7, 0.5)$.
We then compare the following baseliens:
\begin{itemize}
\item Track and Stop (TaS)
\item Track and Stop with mediators' feedback and known policies (TaS-MF-k)
\item Track and Stop with mediators' feedback and unknown policies (TaS-MF-u) 
\item Uniform Sampling (US) over the set of mediator 
\end{itemize}
Both for TaS-MF-k, TaS-MF-u and Uni we assume access to the a set of $E=4$ mediators, whose policies are given by $\bm{\pi}_{e_1} = [0.1, 0.8, 0.1, 0]$, $\bm{\pi}_{e_2} = [0, 0.1, 0.8, 0.1]$, $\bm{\pi}_{e_3} = [0, 0.1, 0.1, 0.8]$, and $\bm{\pi}_{e_4} = [0.2, 0.0, 0.4, 0.4]$. We have tested our algorithm on $4$ different values of $\delta$, namely $\delta = [0.4, 0.1, 0.01, 0.001]$.


We have run $100$ for each algorithm using 100 Intel(R) Xeon(R) Gold 6238R CPU @ 2.20GHz cpus and 256GB of RAM. The rime required for obtaining all results is moderate (less than a day).

Table~\ref{table} shows the result. As we can see, TaS outperforms TaS-MF-k and TaS-MF-u due to the presence of the mediators that limit the controllability over the arm space. Nevertheless, TaS-MF-k and TaS-MF-u are still able to outperform the Uniform Sampling baseline, thus showing the fact that our sequential-decision making strategy is effective at exploiting the set of mediators. Finally, TaS-MF-k and TaS-MF-u performs similar but having access to the knowledge of $\bm{\pi}$, in practice, leads to some advantages (especially for moderate regimes of $\delta$). 

\subsection*{Experiment 2}
We now propose a second experiment that analyzes more deeply the difference in performance between TaS-MF-k and TaS-MF-u. Indeed, it has to be noticed that, whenever the mediators' policies are known to the agent, if identical policies are present within the set, the algorithm can actually avoid querying identical copies of the mediators. \footnote{More generally, the algorithm can remove some $\bm{\pi_e} \in \bm{\pi}$ whenever the set $\widetilde{\Sigma}_k$ does not change.} In other words, TaS-MF-k can be easily modified (without affecting its theoretical guarantees) to remove all copies of identical policies $\bm{\pi_e} \in \bm{\pi}$. In pratice, in scenarios such as the one that we will describe in a moment, this turns out to significantly affect the sample complexity, especially for moderate regime of $\delta$. More precisely, we consider the case where $\bm{\mu} = [5.0, 1.0]$ and $\bm{\pi}$ contains $E=10$ mediators. We consider $\bm{\pi}_{e_1} = [0.01, 0.99]$ and $\bm{\pi}_{e_i} = [0.0, 1.0]$ for all $i \ne 1$. We test both algorithms in the regimes of $\delta$: [$0.4$, $0.1$, $0.01$, $0.001$, $0.0001$, $0.00001$, $0.000001$, $0.0000001$, $0.00000001$, $0.000000001$, $0.0000000001$, $0.00000000001$, $0.000000000001$].

We run $1000$ runs for both algorithms using 100 Intel(R) Xeon(R) Gold 6238R CPU @ 2.20GHz cpus and 256GB of RAM. The rime required for obtaining all results is moderate (less than a day).

Table~\ref{table2} reports the result. As we can see, especially for moderate regime of $\delta$ there is a significant difference between the two algorithms. This is due to the fact TaS-MF-u pays a significant amount of samples for the forced exploration to various mediators that are equivalent, namely $\bm{\pi}_{e_i}$ for all $i \ne 1$. The difference between TaS-MF-k and TaS-MF-u, however, tends to shrink as soon as we decrease the values of $\delta$ (as predicted by our theory). 

\begin{table}[t]
\centering
\begin{tabular}{l|l|l|l|l|}
	& \textbf{TaS-MF-k} & \textbf{TaS-MF-u}   \\
    $\delta=0.4$ & $254.75 \pm 15.01$ & $1120.97 \pm 63.77$ \\
    $\delta=0.1$ & $277.83 \pm 14.63$ & $1105.42 \pm 60.53$ \\
    $\delta=0.01$ & $307.87 \pm 14.01$ & $1154.54 \pm 65.54$ \\
    $\delta=0.001$ & $343.43 \pm 16.66$ & $1121.26 \pm 56.74$ \\
    $\delta=0.0001$ & $395.66 \pm 16.21 $ & $1219.50 \pm 60.55$ \\
    $\delta=0.00001$ & $405.05 \pm 14.89 $ & $1228.79 \pm 63.32$ \\
    $\delta=0.000001$ & $408.33 \pm 15.33 $ & $1264.96 \pm 60.18$ \\
    $\delta=0.0000001$ & $484.49 \pm 17.70$ & $1330.48 \pm 61.40$ \\
    $\delta=0.00000001$ & $493.63 \pm 16.10$ & $1366.83 \pm 61.08$ \\
    $\delta=0.000000001$ & $506.81 \pm 16.94$ & $1347.70 \pm 62.64$ \\
    $\delta=0.0000000001$ & $569.99 \pm 18.65$ & $1514.64 \pm 67.12$ \\
    $\delta=0.00000000001$ & $609.26 \pm 17.10$ & $1434.92 \pm 59.57$ \\
\end{tabular}
\caption{Experiment results for $1000$ runs. The table report mean and $95\%$ confidence intervals of the empirical stopping time.}
\label{table2}
\end{table}

\end{document}